\documentclass[10pt,twocolumn,letterpaper]{article}

\usepackage{iccv}              
\definecolor{iccvblue}{rgb}{0.21,0.49,0.74}
\usepackage[pagebackref,breaklinks,colorlinks,allcolors=iccvblue]{hyperref}
\def\paperID{12214} 
\def\confName{ICCV}
\def\confYear{2025}
\newtheorem{theorem}{Theorem}[section]

\newtheorem{corollary}{Corollary}[section]

\newenvironment{proof}{
  \textbf{Proof.} 
}{
  \hfill\ensuremath{\square}
}
\usepackage{xcolor}
\usepackage{multirow}
\usepackage{footnote} 
\title{Skip-Vision: Efficient and Scalable Acceleration of Vision-Language Models via Adaptive Token Skipping}

\author{Weili Zeng\textsuperscript{$^*$}\\
MoE Key Lab of Artificial Intelligence, AI Institute\\
Shanghai Jiao Tong University\\
{\tt\small zwl666@sjtu.edu.cn}
\and
Ziyuan Huang\\
Ant Group\\
{\tt\small ziyuan.huang@u.nus.edu}
\and
Kaixiang Ji\\
Ant Group\\
{\tt\small kaixiang.jkx@antgroup.com}
\and
Yichao Yan\textsuperscript{$^\dag$}\\
MoE Key Lab of Artificial Intelligence, AI Institute\\
Shanghai Jiao Tong University\\
{\tt\small yanyichao@sjtu.edu.cn}
}
\begin{document}
\maketitle
\makeatletter
\def\@thefnmark{} 
\@footnotetext{\textsuperscript{*}This work was done when Weili Zeng interned at Ant Group.}
\@footnotetext{\textsuperscript{\dag}Corresponding author}
\makeatother
\begin{abstract}
Transformer-based models have driven significant advancements in Multimodal Large Language Models (MLLMs), yet their computational costs surge drastically when scaling resolution, training data, and model parameters. 
A key bottleneck stems from the proliferation of visual tokens required for fine-grained image understanding.
We propose Skip-Vision, a unified framework addressing both training and inference inefficiencies in vision-language models. On top of conventional token compression approaches, our method introduces two complementary acceleration strategies. 
For training acceleration, we observe that Feed-Forward Network (FFN) computations on visual tokens induce marginal feature updates. 
This motivates our Skip-FFN strategy, which bypasses FFN layers for redundant visual tokens. 
For inference acceleration, we design a selective KV-cache removal mechanism that prunes the skipped key-value pairs during decoding while preserving model performance.
Experimental results demonstrate that Skip-Vision reduces training time by up to 35\%, inference FLOPs by 75\%, and latency by 45\%, while achieving comparable or superior performance to existing methods. Our work provides a practical solution for scaling high-performance MLLMs with enhanced efficiency.
\end{abstract}

\section{Introduction}
\label{sec:intro}

Transformer-based large-scale models have significantly advanced progress in Artificial Intelligence (AI)~\cite{clip,gpt,touvron2023llama}. Their remarkable capabilities have driven the emergence of modern Multimodal Large Language Models (MLLMs)~\cite{flamingo,llava,gemini,gpt4}, which demonstrate vision-language competencies comparable to human performance. Empirical evidence suggests that scaling laws continue to be effective, primarily across three dimensions: visual scaling~\cite{llava-improved,llava-uhd,cos}, data scaling~\cite{sphinx,lin2024vila}, and model scaling~\cite{qwen2,li2024llavanext-strong}. However, these scaling methods present substantial efficiency issues, significantly increasing the training and inference burden. 
In a typical LLaVA-1.5 framework~\cite{llava}, where the model receives 576 visual tokens and 40 text tokens, inference using Vicuna-13B~\cite{vicuna} LLM backbone demands 200 trillion FLOPS~\cite{efficient}. These motivate our pursuit of an MLLM architecture that optimally balances efficiency and performance.


\begin{figure}[t]
  \centering
   \includegraphics[width=\linewidth]{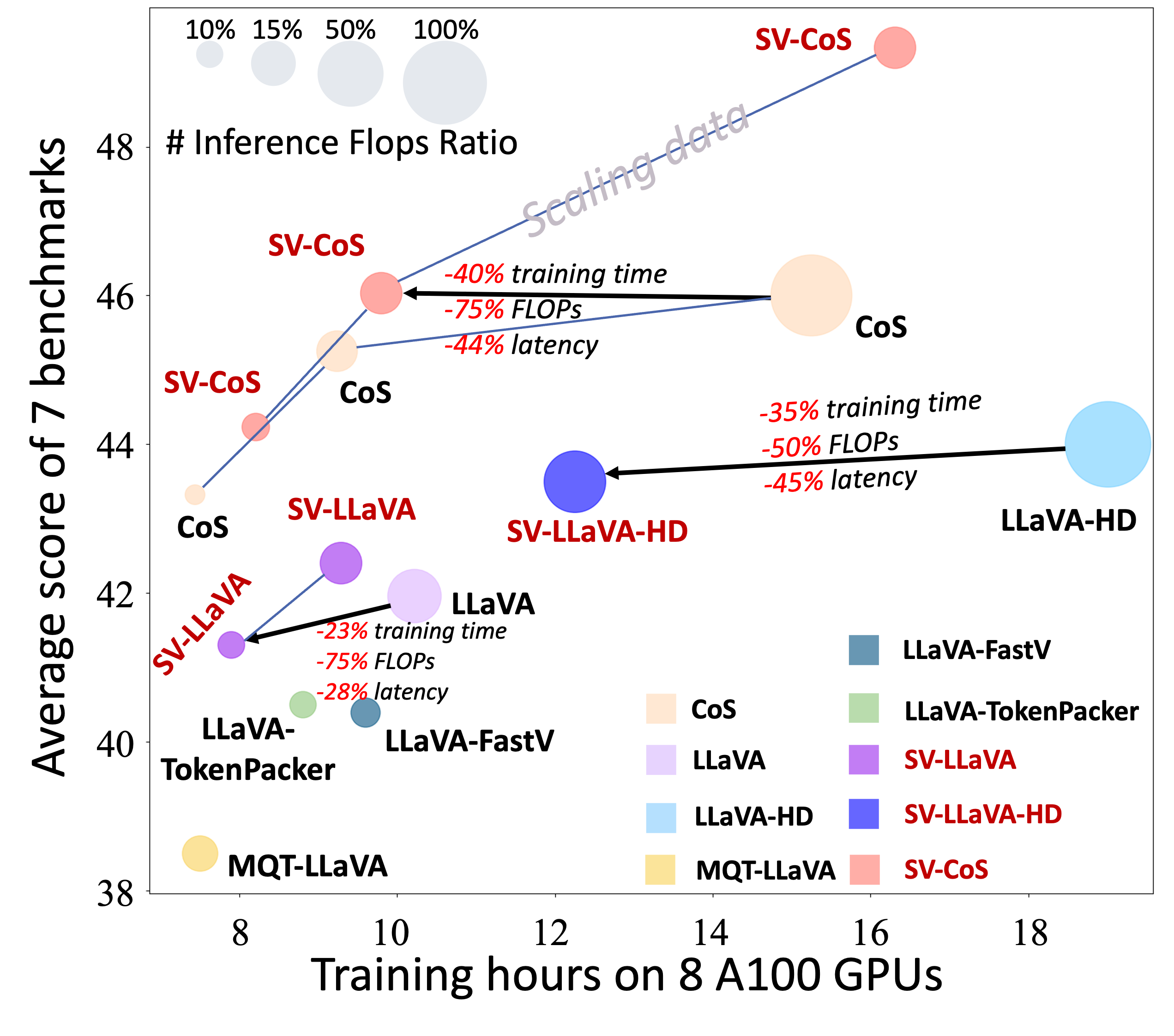}
   \caption{\textbf{Performance-efficiency trade-off curve.} Each circle denotes a model configuration, where our models utilize the Skip-Vision framework, with CoS~\cite{cos}, LLava-HD~\cite{llava-improved} and LLaVA~\cite{llava} serving as baselines. Circle sizes reflect the inference FLOPs ratio. Skip-Vision demonstrate superior performance, scaling effectively with increased FLOPs and data and achieving higher inference efficiency when compared to baselines and other effcient MLLM methods. All methods utilize LLaMA3 8B as the foundational large language model.}
   \label{fig:2}
   \vspace{-3mm}
\end{figure}

Recent advances in Multimodal Large Language Models (MLLMs), such as Chain-of-Sight (CoS)\cite{cos} and LLaVA-Next\cite{llavanext}, have improved performance by increasing visual tokens. However, during the Supervised Fine-Tuning (SFT) stage, the computational cost rises dramatically, leading to diminishing returns and making the scaling unsustainable. As shown in Figure \ref{fig:2}, CoS requires a $165\%$ increase in training resources and a $400\%$ increase in inference computation for a mere $1.7\%$ performance improvement, while LLaVA-Next faces similar inefficiencies.
This challenge emphasizes the need to reduce visual token computation. Approaches like token compression or pruning \cite{tokenpacker,deco,fastv,mqt,voco,auroracap}, while efficient, risk losing critical visual information, revealing a trade-off between efficiency and performance. This work aims to optimize the model's computational efficiency from an architectural perspective, proposing a scalable solution that balances efficiency with visual fidelity.


\begin{table}
  \centering
  \begin{tabular}{c|c|c}
    \hline Llama layer & Complexity & Parameters \\
    \hline FFN & $10.5LNC^2$ & $5.64B (70 \%)$ \\
    Attention & $2LN^2C+1.5NC^2$ & $1.34B (17 \%)$ \\
    Embedding & - & $0.5B (6 \%)$ \\
    \hline
    \end{tabular}
  \caption{\textbf{Computational complexity and parameter statistics for Llama3 8B.} L denotes the number of blocks, N represents the number of input tokens, and C is the feature dimensionality}
  \label{tab:1}
  \vspace{-3mm}
\end{table}

We systematically examine the parameter count and FLOPs distribution across different components of the LLM. As shown in Table \ref{tab:1}, using LLaMA-3~\cite{llama3} as an example, the parameter and FLOPS requirements for the FFN are significantly higher than the attention module. Furthermore, as illustrated in Figure \ref{fig:1}, when the number of visual tokens is large, the majority of computations are concentrated on FFN of these tokens. 
 In particular, we computed the difference in the magnitudes of the features before and after the FFN transformation for different tokens. As illustrated in the figure \ref{fig:norm}, we observe that the updates introduced by the FFN to the visual tokens are significantly smaller than those that affect the text tokens. This is similar to the observation in \cite{Shortv}.
This observation suggests a potential approach: \textbf{could we bypass FFN computations for a substantial portion of visual tokens? }
Furthermore, in the autoregressive training framework of MLLMs, only the final visual token utilizes its last hidden state for next-token prediction, while the other visual tokens primarily serve to convey causal information. During model inference, inspired by~\cite{understanding-mllm}, we observe that most visual information is progressively integrated into the final token through the early causal layers. This raises the possibility of reducing reliance on the full visual scaling KV cache throughout inference, potentially accelerating the inference process.


\begin{figure}[h]
  \centering
   \includegraphics[width=\linewidth]{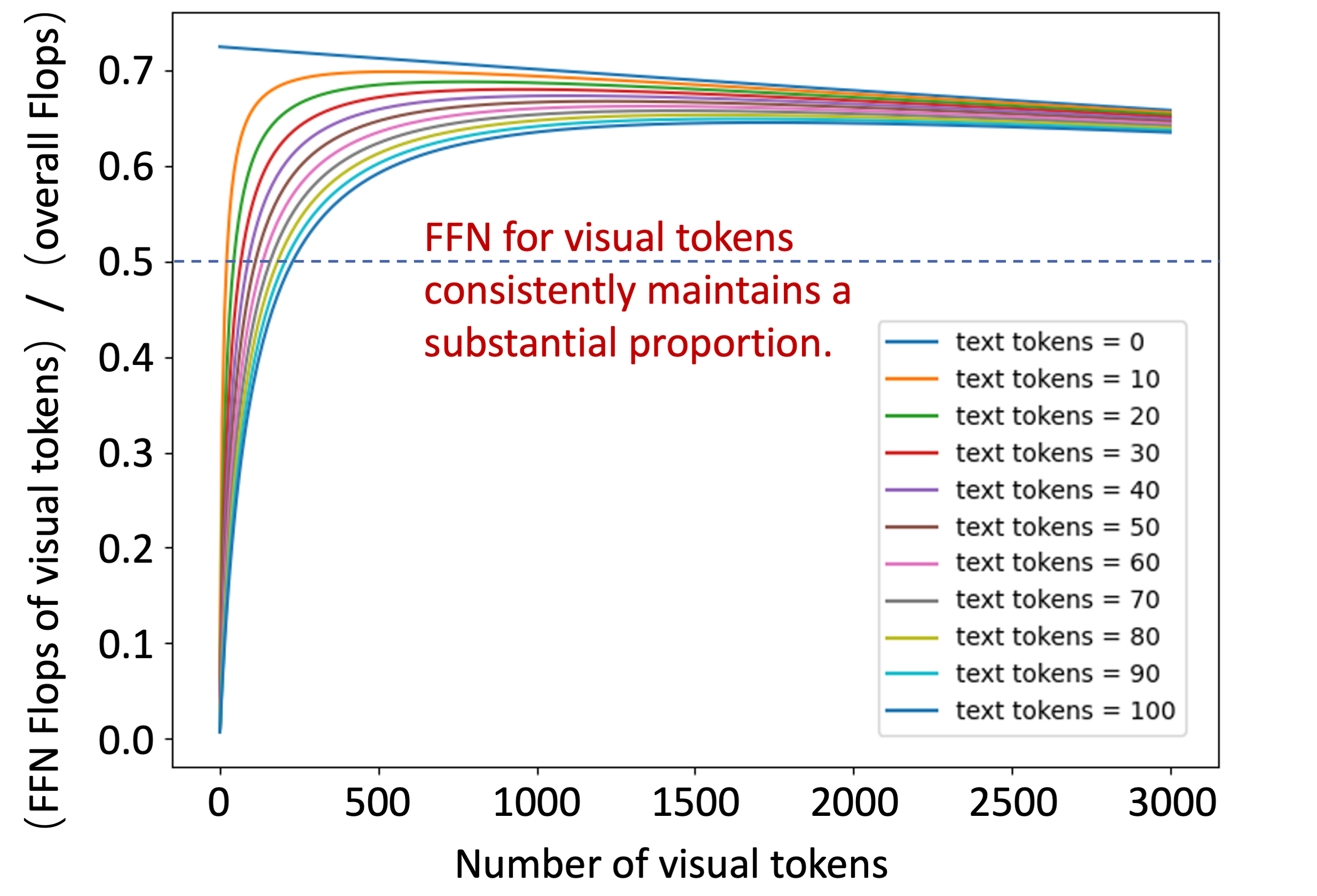}
   \caption{\textbf{Flops ratio of visual tokens in FFN.} When visual tokens greatly outnumber text tokens, computation is predominantly consumed by the FFN of visual tokens. Reducing their passage through the FFN can thus markedly decrease computational costs.}
   \label{fig:1}
   \vspace{-3mm}
\end{figure}

\begin{figure*}[t]
    \centering
    \subfloat[FFN updates for CoS]{%
        \includegraphics[width=0.33\textwidth]{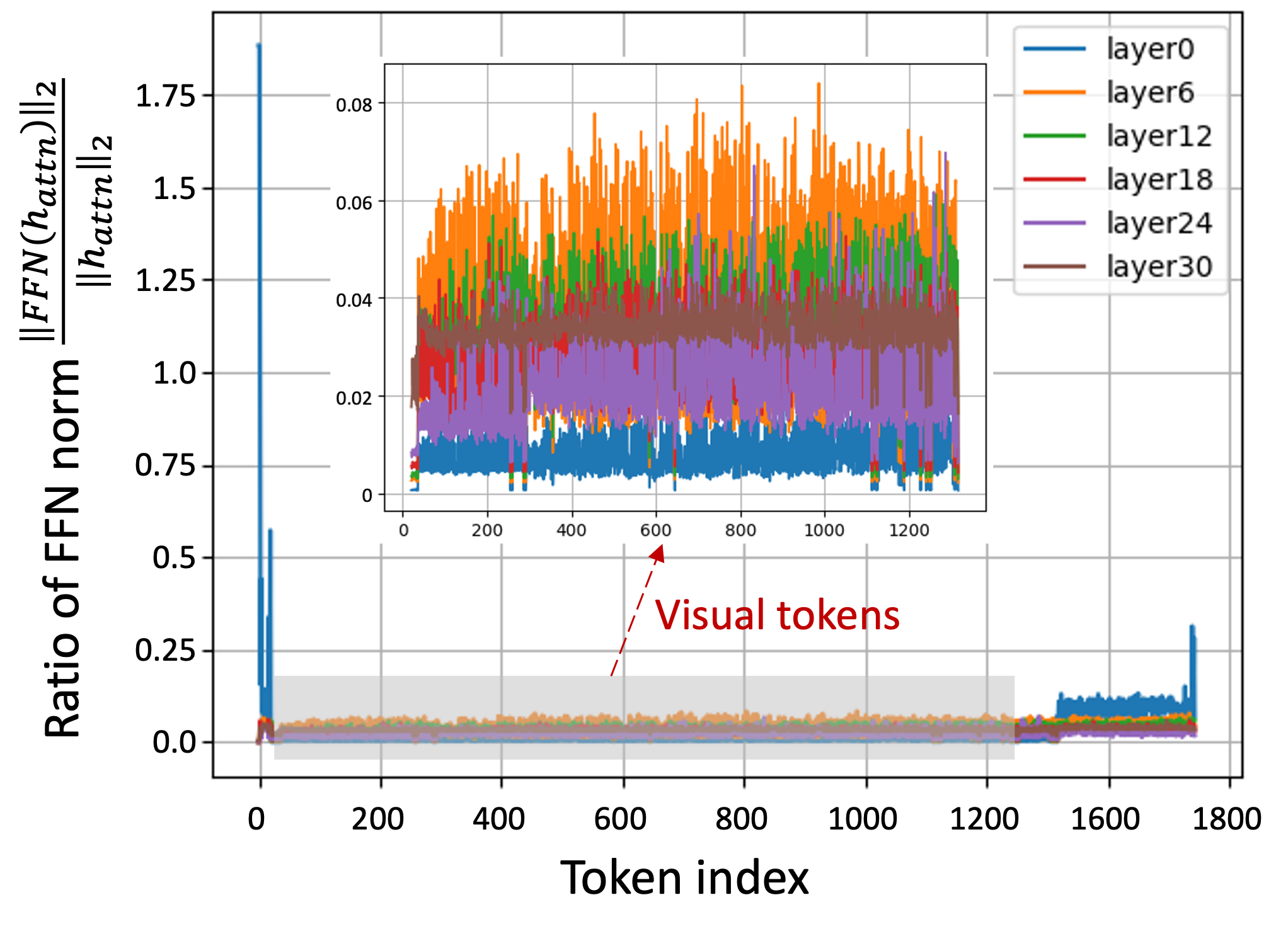}
    }
    \subfloat[FFN updates for LLaVA-HD]{%
        \includegraphics[width=0.33\textwidth]{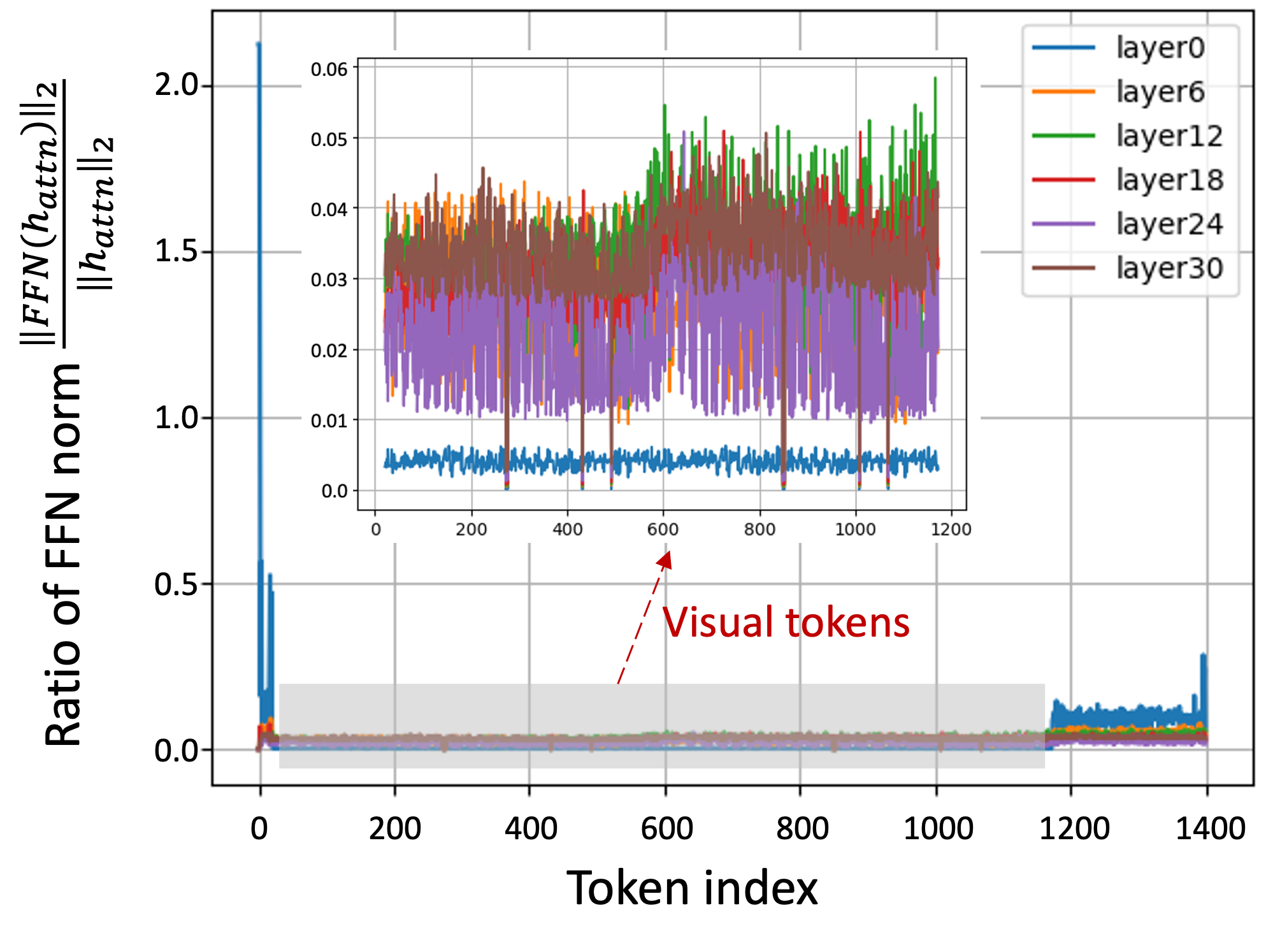}
    }
    \subfloat[FFN updates for LLaVA]{%
        \includegraphics[width=0.33\textwidth]{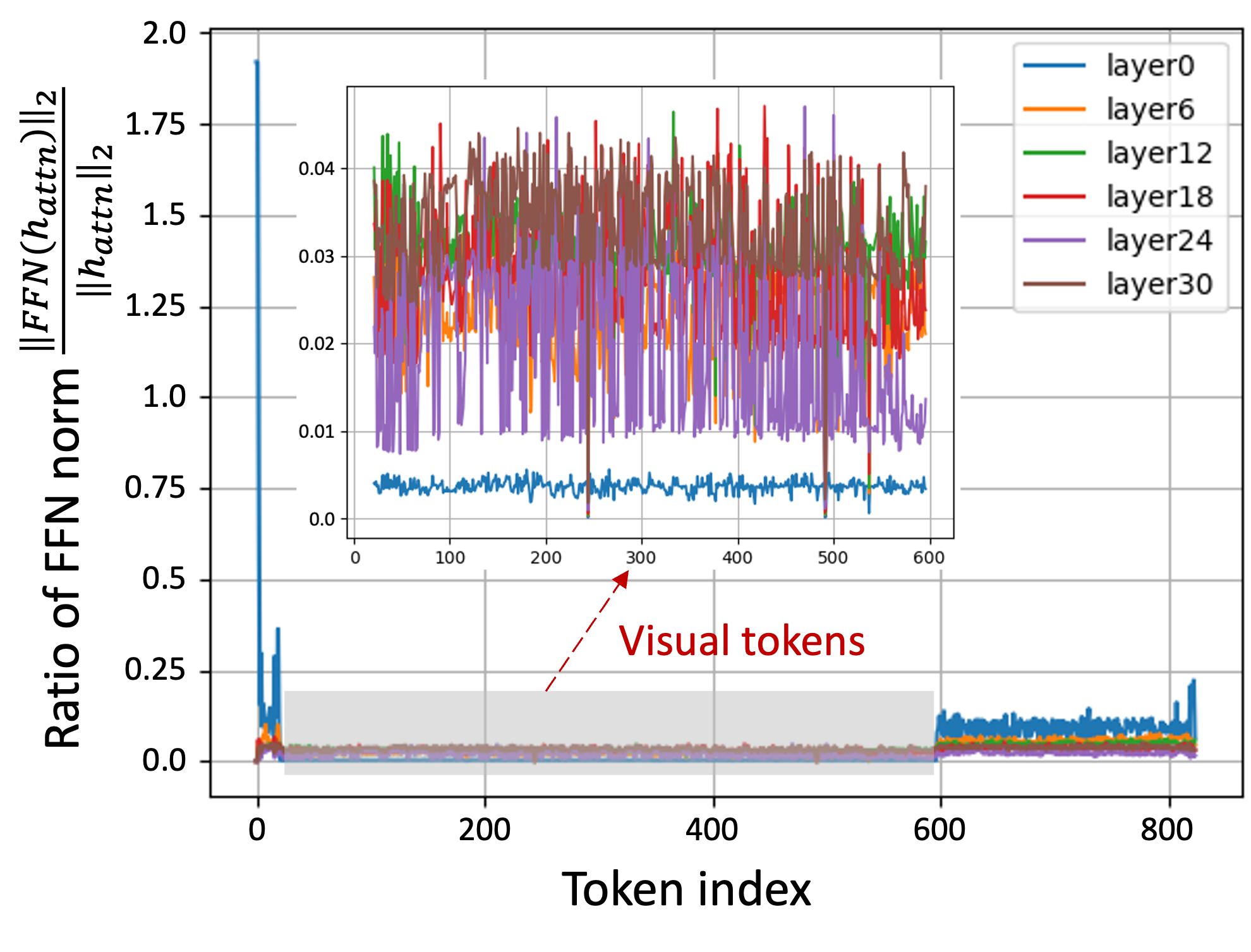}
    }
    \caption{\textbf{FFN imapct for different MLLMs.} We evaluate FFN impact by computing the feature modulus ratio before ($\|h_{\text{attn}}\|_2$) and after FFN ($\|\text{FFN}(h_{\text{attn}})\|_2$). The FFN introduces significantly smaller updates to visual tokens compared to text tokens.}
    \label{fig:norm}
    \vspace{-3mm}
\end{figure*}


Building on these insights, we propose Skip-Vision, a novel and flexible architecture designed for more efficient scaling (see Figure \ref{fig:2} for a comparison). Visual tokens derived from smaller window/scale sizes constitute approximately $80\%$ of the total visual tokens and account for most of the computational overhead. To mitigate this, as shown in Figure \ref{fig:3}, we allow these tokens to skip the Feed-Forward Network (FFN) computation, significantly reducing overall computational costs. To address redundancy within skipped vision tokens, we can apply a token merging strategy for further simplification.
During inference, skipped tokens are removed from the KV cache after the pre-fill phase, accelerating subsequent computation. Empirical results demonstrate that Skip-Vision reduces training time by up to $35\%$, FLOPs of inference by $75\%$ and inference latency by $45\%$, while maintaining or even improving performance. These efficiency gains provide a scalable solution for large-scale multimodal learning, enabling better data utilization and enhanced model scalability.

In summary, Skip-Vision is designed to be seamlessly integrated into the standard SFT pipeline of MLLMs without introducing additional retraining or decoupled modules. It directly modifies the transformer’s computation flow, offering a practical and theoretically grounded acceleration solution for MLLM training and inference jointly. Our contributions are as follows:
\begin{itemize}
\item We introduce a novel and efficient framework, using token merge and skip FFN strategy during training to reduce redundant computation for visual tokens.
\item In inference, our framework employs a skip KV-cache mechanism that removes skip-FFN visual tokens from the KV-cache, enhancing efficiency.
\item We present a theoretical analysis of the performance error in Skip-Vision and establish an error bound, which will be empirically validated through experiments.
\item Experiments show our model's superior efficiency, effective data scaling, and performance on par with state-of-the-art models of similar scale.
\end{itemize}

\section{Related work}
\subsection{Vison-Language models}
Rapid progress in large language models (LLMs)~\cite{llama3,vicuna,qwen2llm} has laid a solid foundation for the emergence of large-scale vision-language models (VLMs)~\cite{flamingo,blip2,qwenvl,shikra}. Flamingo~\cite{flamingo} was a pioneering effort, integrating pre-trained image encoders with LLMs to handle multimodal data, thereby allowing comprehensive understanding and reasoning across modalities. Following this, the major models developed by leading corporations, such as GPT-4v~\cite{gpt4} and Gemini-1.5-Pro~\cite{gemini}, have led the progress in this domain. Using proprietary data sets and undisclosed training methodologies, these models have elevated multimodal intelligence to unprecedented levels. 

At the same time, the open-source community has worked diligently to stay abreast of these advancements. The LLava series models~\cite{llava,llava-improved,llava-uhd,li2024llavanext-strong} exemplify a key strategy in this domain by mapping visual features directly into the input embedding space of the language model, integrating them as input tokens. However, this approach generates a substantial volume of visual tokens, frequently exceeding the number of text tokens, creating efficiency challenges.
InternVL-1.5~\cite{InternVL} adopts dynamic high-resolution techniques, segmenting large images into smaller fragments for processing and scaling them accordingly. MiniCPM-V~\cite{minicpm} utilizes a perceiver resampler structure with a single layer of cross-attention to compress visual tokens. Models such as Vary~\cite{vary}, SPHINX~\cite{sphinx}, Cambrian-1~\cite{cambrian}, and BRAVE~\cite{brave} leverage multi-visual encoder architectures to strengthen their visual processing capabilities. Furthermore, a crucial factor contributing to the effectiveness of modern VLMs is instruction fine-tuning~\cite{instructblip,minigpt,llava-improved}, enabling these models to function as conversational agents and interact with users through natural, human-like dialogue.
\subsection{Efficient multimodal large language models}

To enhance the efficiency of vision-language models, one of the primary and widely employed approaches in contemporary research involves reducing the size of the backbone models~\cite{mobilevlm,tinygpt,tinyllava,bunny}. Simultaneously, given that visual tokens contribute significantly to computational demands, research focused on compressing or refining these tokens has garnered increasing attention. Such methods are often coupled with the design of vision-language projectors. For instance, techniques like the Perceiver Resampler~\cite{perceiver,qwenvl,flamingo} and Q-Former~\cite{blip2} utilize transformer-based mechanisms to consolidate visual tokens into more compact query sets. In this way, the transformer outputs corresponding to the positions of the learnable latent queries serve as the aggregated representation of the visual features~\cite{efficient}. Honeybee~\cite{honeybee} propose two visual projectors, namely C-Abstractor and D-Abstractor. LLaVA-PruMerge~\cite{prumerge} and MADTP~\cite{madtp} introduce adaptive methods for reducing visual tokens. In addition, while the large language model (LLM) field has employed various techniques for token reduction to accelerate inference and compress KV cache~\cite{lminfinite,h2o}, research in this area remains limited within the MLLM domain.

\section{Preliminaries}
\begin{figure*}[t]
  \centering
   \includegraphics[width=\linewidth]{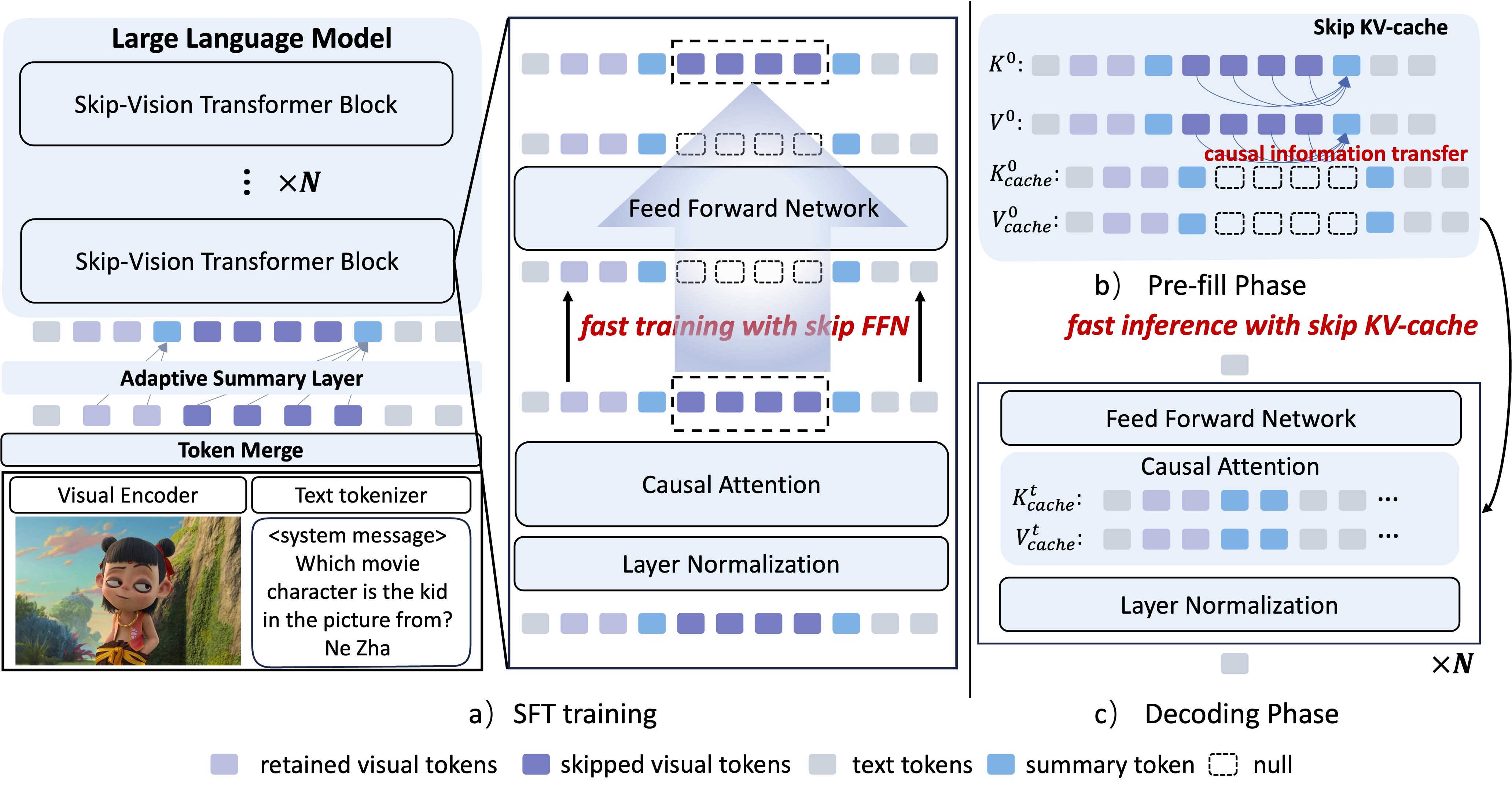}
   \caption{\textbf{The framework of Skip-Vision.} a) While visual scaling enriches visual information, it also increases computational overhead. Skip-Vision uses a skip-FFN strategy during training to reduce redundant computation for visual tokens. The numerous skipped tokens will be limited to the attention layer and bypass FFN. b) At the beginning of inference, Skip-Vision will remove skip-FFN visual tokens from the initial KV-cache, enhancing efficiency. c) During inference, skip attention leverages the skip KV-cache to accelerate generation.}
   \label{fig:3}
   \vspace{-3mm}
\end{figure*} 
\subsection{Decoder-only LLM}
\label{pre}
Each Transformer layer in a decoder-only LLM primarily consists of a self-attention layer and a feed-forward neural network. To facilitate autoregressive inference, the use of a KV-cache has been developed to accelerate the decoding process.

\textbf{Self-Attention Mechanism}: 
The self-attention mechanism enables each token in a sequence to selectively attend to other tokens, capturing contextual dependencies. Each token generates a query $(Q)$, key $(K)$, and value $(V)$ vector. Attention scores are calculated by the scaled dot-product:
\begin{equation}
  \operatorname{Attention}(Q, K, V)=\operatorname{softmax}\left(\frac{Q K^T}{\sqrt{d_k}}\right) V,
  \label{eq1}
\end{equation}
where $d_k$ is the dimension of $K$. This formulation allows each token to incorporate information from others based on relevance, making self-attention critical in capturing relationships across tokens.

\textbf{Feed-Forward Neural Network}: The feed-forward neural network enhances token representations following the self-attention layer. 
The operation can be expressed as:
\begin{equation}
\operatorname{FFN}(x)=\operatorname{Activation}\left(x W_1+b_1\right) W_2+b_2,
\end{equation}
where $x$ is the input, $W_1$ and $W_2$ are weight matrices, and $b_1$ and $b_2$ are biases. This architecture allows the model to capture complex patterns and improve the expressiveness of token representations.

\textbf{KV-cache}: The key-value cache (KV-cache) stores key $(K)$ and value $(V)$ representations from previous self-attention computations, enabling efficient reuse during inference. This approach avoids redundant calculations for earlier tokens, improving inference speed.

The attention score is computed using $\operatorname{Attention}\left(Q, K^t_{\text {cache }}, V^t_{\text {cache }}\right)$,
where $Q$ is the query for the current token.
At each time step $t$, the cache is updated as:
\begin{equation}
K_{\text {cache }}^{(t)}=\left[K_{\text {cache }}^{(t-1)} ; K_t\right], \quad V_{\text {cache }}^{(t)}=\left[V_{\text {cache }}^{(t-1)} ; V_t\right],
\end{equation}
where $K_t$ and $V_t$ are the current token's key and value. This mechanism enhances computational efficiency and accelerates token generation.

\section{Method}
\label{sec:method}
\subsection{Skipped token selection and token merge} 
Before applying the skip FFN strategy, we first identify the skipped and retained visual tokens. The retained tokens should preserve most essential visual information, while the skipped tokens serve as complementary details, ensuring efficiency without significant loss of information.

Fortunately, in MLLM architectures, we can naturally distinguish retained and skipped tokens based on global and local context tokens. Global context tokens are retained, while local context tokens are skipped. In architectures like LLaVA-HD, global tokens originate from the resized full image, whereas local tokens come from smaller image patches. In CoS, global tokens correspond to larger window sizes, while local tokens are derived from smaller windows. Since LLaVA lacks a clear distinction between global and local tokens, inspired by~\cite{visionzip}, we select the top-$n$ tokens most similar to the CLS token as retained tokens, categorizing the rest as skipped tokens.

 To further reduce the computational burden, it is necessary to refine the visual tokens to eliminate excessive redundancy. As shown in Figure \ref{fig:4}, we conducted a similarity analysis of the visual tokens in CoS at different levels, calculating their similarity density. We found that the smaller the granularity of the visual tokens, the higher their average similarity, while their token count is significantly greater than that of other granularities. Consequently, inspired by ToMe~\cite{tome}, we employ a token merge method to reduce redundancy in skipped tokens, as detailed below:
\begin{itemize}
\item \textbf{Compute cosine similarities:} calculate the cosine similarity between each pair of tokens in the set, obtaining an average similarity score for each token.
\item \textbf{Select top $k$ distinctive Tokens:} sort tokens by average similarity in ascending order and retain the top $k$ tokens with the lowest average similarity, representing the most distinctive information.
\item \textbf{Merge remaining tokens:} for the remaining $n-k$ tokens, merge each one with the most similar token among the top $k$, thus consolidating information efficiently.
\end{itemize}

\subsection{Training speed-up: skip FFN}
\label{skip-ffn}

As depicted in Figure \ref{fig:3}a, we propose the FFN skip strategy: The retained visual tokens, which are few in number, proceed conventionally through all decoder layers of the LLM, while the numerous skipped visual tokens will be limited to the self-attention layers of each transformer block. For a given layer $l$,  the output of retained tokens $h_{\text{retained}}^{(l)}$ is the sum of the self-attention output $h_{\text{attn}}^{(l)}$ and the FFN-transformed features $\text{FFN}^{(l)}(h_{\text{attn}}^{(l)})$:
\begin{equation}
h_{\text{retained}}^{(l)}=h_{\text{attn}}^{(l)}+\text{FFN}^{(l)}(h_{\text{attn}}^{(l)}).
\end{equation}
The output of the skipped tokens is:
\begin{equation}
h_{\text{skipped}}^{(l)}=h_{\text{attn}}^{(l)}.
\end{equation}

However, bypassing the FFN layers presents a non-trivial challenge. In the autoregressive training process, the final token of the vision sequence is used to predict the initial token of the text sequence, necessitating a specialized design for this terminal vision token. To facilitate this transition effectively, it must pass through the FFN layers, which are enriched with stored linguistic information~\cite{mass,dissecting}.

To address this, as shown in Figure \ref{fig:3}, we introduce adaptive summary tokens to efficiently consolidate visual information. This token is generated by the Adaptive Summary Layer before the token sequence is input into the LLM. The layer comprises a simple linear transformation, and its computation is defined as follows:
\begin{equation}
x^s=(\operatorname{softmax}(X \cdot W^T))^T \cdot X,
\end{equation}
where $X$ denotes the retained or skipped visual tokens, $x^s \in \mathbb{R}^{1 \times C}$ is the summary token, and $W \in \mathbb{R}^{1 \times C}$ is the weight matrix. Leveraging this layer, we concatenate a summary token on both ends of the skipped visual tokens before their input to the LLM. The first summary token aggregates essential information from retained tokens, mitigating potential loss from extended token sequences. The second summary token integrates prior sequence information and plays a key role in predicting the next text token.
\begin{figure}[t]
  \centering
   \includegraphics[width=0.9\linewidth]
   {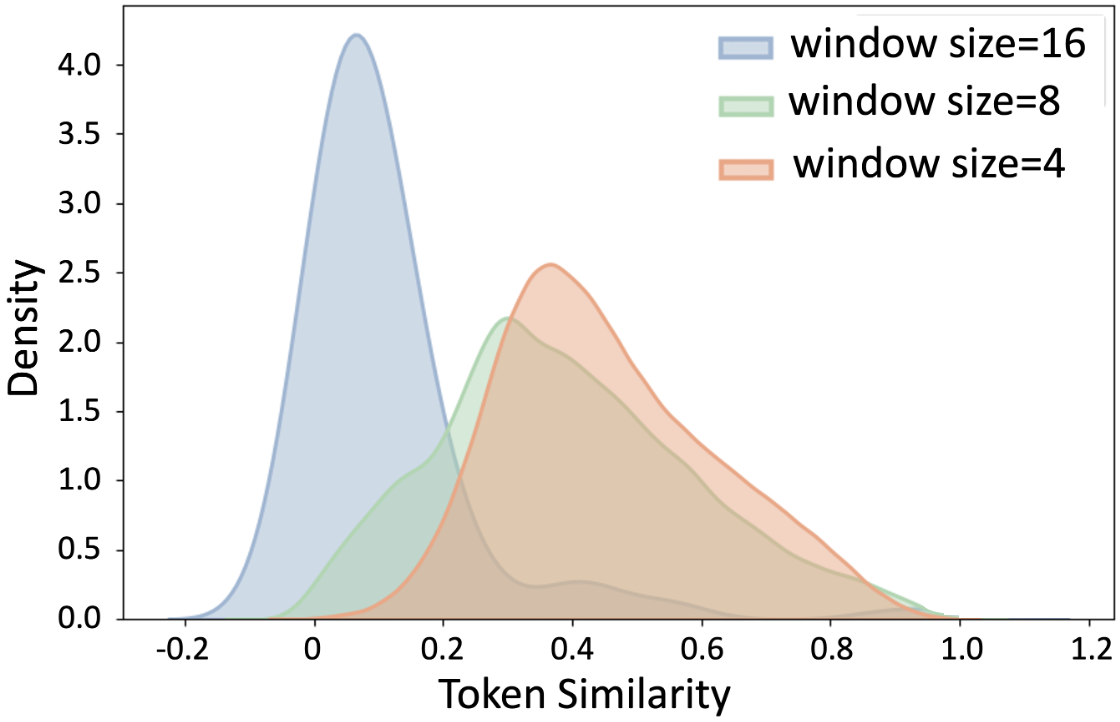}
   \caption{\textbf{Similarity density statistics.} We conducted a similarity density analysis in CoS, examining the visual tokens derived from various window scale. The results indicate that as the window scale decreases, the similarity between visual tokens becomes more pronounced, reflecting higher token redundancy.
   }
   \vspace{-5mm}
   \label{fig:4}
   \vspace{-2mm}
\end{figure}

\subsection{Inference speed-up: skip KV-cache}
After training, our Skip-Vision framework incorporates skip KV-cache method to further improve efficiency:
\begin{equation}
K_{\text {skip-cache }}^{(0)}=K_{\text {cache }}^{(0)} \setminus K_{skip}, \quad V_{\text {skip-cache }}^{(0)}=V_{\text {cache }}^{(0)} \setminus V_{skip},
\label{skip}
\end{equation}
Equation \ref{skip} indicates that visual tokens skipped by FFN are removed from the KV cache after the pre-filling stage. As depicted in Figures \ref{fig:3} b) and c), given the autoregressive nature of LLMs, earlier token information is progressively transferred to later tokens. Skip FFN further enhances this integration by suppressing the attention weights of skipped tokens, forcing the model to focus on key tokens and improving information aggregation efficiency. Empirical results show that, due to the inclusion of summary tokens, this omission does not compromise model performance, underscoring a distinctive advantage of our framework for more efficient inference.

\section{Theoretical analysis of Skip-Vision}
\label{sec:analysis}
At the core of this analysis lies the layer error incurred when bypassing the FFN layer. For a given layer $l$,  the original output $h_{\text{original}}^{(l)}$ is :
$h_{\text{original}}^{(l)}=h_{\text{attn}}^{(l)}+\text{FFN}^{(l)}(h_{\text{attn}}^{(l)}).$
When skipping the FFN, the output becomes:
$
h_{\text{skip}}^{(l)}=h_{\text{attn}}^{(l)}.
$
The per-layer skipping error is:
\begin{equation}
\epsilon^{(l)}=\|h_{\text{original}}^{(l)}-h_{\text{skip}}^{(l)}\|_2=\|\text{FFN}^{(l)}(h_{\text{attn}}^{(l)})\|_2.
\end{equation}
For redundant tokens, such as those in homogeneous image regions, this error is negligible $(\epsilon^{(l)}\approx0)$ due to minimal feature transformations by the FFN.

However, errors propagate through subsequent layers, amplified by the recursive nature of transformer architectures. Leveraging Lipschitz continuity assumptions for self-attention and FFN operations ($L^{(l+1)}_{attn}$ and $L^{(l+1)}_{FFN}$), the cumulative error at layer $l+1$ is bounded by:
\begin{equation}
\epsilon^{(l + 1)}\leq (L^{(l+1)}_{\text{attn}}+L^{(l+1)}_{\text{FFN}})\cdot\epsilon^{(l)}+\epsilon_{\text{skip}}^{(l + 1)},
\end{equation}
where $(\epsilon_{\text{skip}}^{(l + 1)}$ represents new errors from skipping deeper layers $(l + 1)$.

Over $L$ layers, the total error telescopes to:
\begin{equation}
\epsilon_{\text{total}} \leq \sum_{l = 1}^{L}\epsilon_{\text{skip}}^{(l)} \cdot \prod_{i = 1}^{L-l}(L^{(i+1)}_{\text{attn}}+L^{(i+1)}_{\text{FFN}}).
\end{equation}

\begin{theorem}
\label{cor:1}
\textbf{Bounded Lipschitz.}
If the spectral norms of ${W_1}$, ${W_2}$, ${W_K}$, ${W_Q}$, and ${W_V}$ are bounded by 1, then:
\begin{itemize}
    \item $\text{L}(\text{attn})\leq1/\sqrt{d_k}$
    \item $\text{L}(\text{FFN})\leq1$
\end{itemize}
\end{theorem}
Follow Theorem \ref{cor:1}, we assume Lipschitz constants $L_{\text{attn}}+L_{\text{FFN}} \leq \gamma$ and skipping errors $\epsilon_{\text{skip}}^{(l)} \leq \epsilon$, the total error is scaled to:
\begin{equation}
\epsilon_{\text{total}} \leq \epsilon\cdot\frac{\gamma^{L}-1}{\gamma - 1},
\label{bound}
\end{equation}
Equation \ref{bound} establishes that the skip error is bounded when $\gamma<1$, provided the model is trained with modern regularization techniques. This ensures that the MLLM remains less sensitive to the effects of skipping.

This error also impacts the KL divergence between the original and skipped outputs, bounded by:
\begin{equation}
\mathcal{D}_{\text{KL}}(p_{\text{skip}}\parallel p_{\text{original}})\leq\frac{1}{2\sigma^{2}}\cdot\epsilon_{\text{total}}^{2},
\end{equation}
where $\sigma^{2}$ is the variance of the logits. Further integrating feature similarity errors $(\epsilon_{\text{sim}} = O(\sqrt{1 - \theta}))$ from low-attention tokens, the final bound becomes:
\begin{equation}
\mathcal{D}_{\text{KL}}\leq\frac{1}{2\sigma^{2}}\cdot\left(\epsilon_{\text{total}}+\epsilon_{\text{sim}}\right)^{2}.
\end{equation}
Practically, this analysis motivates a layer-wise skipping strategy, alongside token selection and merge based on feature similarity $(\theta)$. \textbf{More analysis and the proof of Theorem \ref{cor:1} are provided in the Appendix.}

\section{Experiments}
\subsection{Benchmarks}
To comprehensively evaluate our model, we conducted experiments on eight different benchmarks. These include MME~\cite{mme}, which assesses the perceptual and cognitive capabilities of multimodal language models; TextVQA~\cite{textvqa}, focused on fine-grained visual question answering; MMBench~\cite{mmbench} and MMStar~\cite{mmstar} for general diagnostic capabilities; MMMU~\cite{mmmu}, designed to test STEM-related reasoning; MathVista~\cite{mathvista}, which evaluates mathematical problem solving skills; OCRBench~\cite{ocr}, for optical character recognition tasks; and MMVet~\cite{mmvet}, used for subjective assessment. Taken together, these benchmarks offer a comprehensive framework for assessing the efficacy of the model in various multimodal challenges.
\subsection{Efficiency evaluation}
We systematically evaluate both training time ($hours$), inference FLOP computation and inference latency ($ms$) on single A100 GPU. For training efficiency, we report actual GPU hours measured on consistent hardware configurations. For inference FLOPs, we adopt the methodology from FastV~\cite{fastv}, specifically tracking the computational demands of processing visual tokens. In the LLama architecture, we calculate FLOPs for the causal attention and feed forward network modules as $4 N C^2+2 N^2 C+3 N C M$ where $N$ is the token count, $C$ denotes the hidden state dimension, and $M$ is the intermediate FFN dimension. Skip-Vision framework adjusts token counts dynamically between the attention and FFN modules, allowing us to refine FLOP calculations as $4 N_1 C^2+2 N_1^2 C+3 N_2 C M$, where 
$N_1$ and $N_2$ represent the visual token counts processed in the causal attention and FFN layers, respectively.

\subsection{Data setup}
During pretraining, we adopt the experimental settings of LLaVA and CoS, utilizing LLaVA-558K~\cite{llava} and 65 million image-text pairs for respective comparisons. For the Supervised Fine-Tuning phase, we use LLaVA-665K~\cite{llava} for all comparative experiments. 
By leveraging the computational efficiency of Skip-Vision framework, the resources saved can be redirected toward scaling the dataset to further enhance model performance. To this end, we extend our dataset to SV-1M, comprising 1 million samples, demonstrating the substantial performance improvements achievable through data expansion. To further assess our approach against state-of-the-art models of similar scale, we extend the dataset to SV-9M, underscoring the scalability and potential of our method. Detailed dataset statistics are provided in the Appendix.
\subsection{Training setup}
\begin{table*}
  \centering
  \setlength\tabcolsep{0.6pt}
  \begin{tabular}{lccccccccccccc}
\hline \multicolumn{1}{c}{} & MME & Textvqa & MMB & MMVet & MMMU & MathV & OCRB & MMStar & Avg & Hours & FLOPs & Latency \\
\hline \textbf{LLaVA setting:} & & & & &  & & & & & & & \\
LLaVA~\cite{llava} & 1535 & 58.6 & 72.2 & 32.8 & 39.6 & 21.3 & 32.7 & 36.5 & 42.0 & 10.2 & 100\% & 125  \\
LLaVA-FastV~\cite{fastv} & 1453 & 56.1 & 71 & 32.4 & 37.3 & 20.6 & 28.8 & 36.3 & 40.4 & 9.6 & 30\% & 103 \\
MQT-LLaVA~\cite{mqt} & 1487 & 53.6 & 69 & 28.2 & 37.8 & 16.8 & 28.1 & 36.5 & 38.6 & 7.5 & 44\% & 106 \\
LLaVA-Tokenpacker~\cite{tokenpacker} & 1538 & 56.2 & 72.3 & 32.1 & 39.0 & 19.3 & 29.8 & 35.0 & 40.5  & 8.8 & 24\% & 97 \\
SV-LLaVA & 1519 & 56.5 & 72.5 & 32.4 & 38.8 & 21.2 & 30.7 & 36.5 & 41.3 & 7.9 & 25\% & 90 \\
 $(N_r=100, N_s=156)$ & \textcolor{blue}{$-1\%$} & \textcolor{blue}{$-3.6\%$} & \textcolor{red}{$+0.4\%$} & \textcolor{blue}{$-1.2\%$} & \textcolor{blue}{$-2\%$} & \textcolor{blue}{$-0.5\%$} & \textcolor{blue}{$-3.6\%$} & \textcolor{blue}{$-6.1\%$} & \textcolor{blue}{$-1.7\%$} & \textcolor{red}{$-22.5\%$}& \textcolor{red}{$-75\%$} & \textcolor{red}{$-28\%$}\\
SV-LLaVA & 1530 & 57.3 & 72.7 & 35.6 & 39.2 & 21.4 & 32.0 & 40.0 & 42.6 & 9.0 & 60\% & 103 \\
$(N_r=256, N_s=320)$ & \textcolor{blue}{$-0.3\%$} & \textcolor{blue}{$-2.2\%$} & \textcolor{red}{$+0.7\%$} & \textcolor{red}{$+8.5\%$} & \textcolor{blue}{$-1\%$} & \textcolor{red}{$+0.5\%$} & \textcolor{blue}{$-2.1\%$} & \textcolor{red}{$+9.6\%$} & \textcolor{red}{$+1.4\%$} & \textcolor{red}{$-11.8\%$} & \textcolor{red}{$-40\%$} & \textcolor{red}{$-17.6\%$}\\
\hline \textbf{LLaVA-HD setting:} & & & & & & & & & & & & \\
LLaVA-HD~\cite{llava-improved} & 1533 & 64 & 71.8 & 35.3 & 39.3 & 20 & 37.4 & 40.5 & 44.0 & 19 & 100\% & 450 \\
SV-LLaVA-HD & 1534 & 61.2 & 73.1 & 35.9 & 39.7 & 21.2 & 34.1 & 39.3 & 43.5 & 12.2 & 50\% & 250 \\
$(N_r=576, N_s=576)$ & \textcolor{red}{$+0\%$} & \textcolor{blue}{$-4.3\%$} & \textcolor{red}{$+1.8\%$} & \textcolor{red}{$+1.7\%$} & \textcolor{red}{$+1.5\%$} & \textcolor{blue}{$-6\%$} & \textcolor{blue}{$-8.8\%$} & \textcolor{blue}{$-2.9\%$} & \textcolor{blue}{$-1.1\%$} & \textcolor{red}{$-35.8\%$}& \textcolor{red}{$-50\%$} & \textcolor{red}{$-44.4\%$}\\
\hline \textbf{CoS setting:} & & & & & & & & & & & & \\
CoS~\cite{cos} & 1585 & 64.4 & 77.1 & 39.4 & 39.2 & 21.5 & 39.2 & 41.2 & 46.0 & 15.3 & 100\% & 160 \\
SV-CoS & 1563 & 63.7 & 76.5 & 41.7 & 40.3 & 21.2 & 37 & 41.9 & 46.0 & 9.8 & 25\% & 90 \\
$(N_r=272, N_s=256)$ & \textcolor{blue}{$-1.4\%$} & \textcolor{blue}{$-1.1\%$} & \textcolor{blue}{$-0.8\%$} & \textcolor{red}{$+5.8\%$} & \textcolor{red}{$+2.8\%$} & \textcolor{blue}{$-1.4\%$} & \textcolor{blue}{$-5.6\%$} & \textcolor{red}{$+1.7\%$} & \textcolor{red}{$-0\%$} & \textcolor{red}{$-40\%$} & \textcolor{red}{$-75\%$} & \textcolor{red}{$-43.8\%$}\\
\hline 
SV-CoS (SV-1M)  & 1569 & 67.1 & 77.1 & 43.3 & 40.7 & 22.0 & 50.7 & 44.4 & 49.3 & 16.3 & 25\% & 90 \\
\hline
\end{tabular}
  \caption{\textbf{Performance and efficiency evaluation.} We evaluate Skip-Vision on LLaVA, LLaVA-HD, and CoS and compare it with state-of-the-art efficiency optimization models under the LLaVA setting. $N_r$ and $N_s$ denote the number of retained and skipped tokens, respectively.}
  \label{tab:3}
  \vspace{-3mm}
\end{table*}

\begin{table*}
  \centering
  \setlength\tabcolsep{4pt}
  \begin{tabular}{lcccccccccc}
\hline \multicolumn{1}{c}{} & LLM & MME & Textvqa & MMB & MMVet & MMMU & MathV & OCRB & MMStar \\
\hline Open-source models in 8B tier: & & & & & & & & \\
mPLUG-Owl3~\cite{mplug-owl3} & Qwen2 & - & 69.0 & \underline{77.6} & 40.1 & - & \textbf{65.0} & - & 40.1 \\
Cambrian~\cite{cambrian} & LLaMA3 & 1547 & \underline{71.7} & 75.9 & - & 42.7 & 49.0 & 62.4 & - \\
LLaVA~\cite{llava} & LLaMA3 & 1535 & 58.6 & 72.2 & 32.8 & 39.6 & 21.3 & 32.7 & 36.5 \\
Mini-Gemini-HD~\cite{mini} & LLaMA3 &\textbf{1606} & 70.2 & 72.7 & - & 37.3 & 37.0 & 47.7 & - \\
LLaVA-Next~\cite{li2024llavanext-strong} & LLaMa3 & \underline{1604} & 64.6 & 72.1 & - & 41.7 & 36.3 & 49.0 & - \\
Ovis~\cite{ovis} & LLaMA3 & - &- & 77.4 & - & \underline{44.7} & 40.8 & - & \underline{49.5} \\
MiniCPM-V2.5~\cite{minicpm} & LLaMA3 & - &\textbf{76.6} & 77.6 & - & \textbf{45.8} & 54.3 & \textbf{72.5} & - \\
\hline SV-CoS (SV-9M) & LLaMA3 & 1550 & 69.1 & \textbf{78.4} & \textbf{51.7} & 43 & \underline{61.3} & \underline{69.4} & \textbf{52.3} \\
\hline
\end{tabular}
  \caption{\textbf{MLLM evaluation.} We present the performance of SV-CoS on SV-9M, comparing it against the current SOTA models of a similar scale. The highest scores are highlighted in \textbf{bold}, while the second-highest scores are indicated with \underline{underlines}.}
  \label{tab:2}
  \vspace{-4mm}
\end{table*}

\begin{table*}
  \centering
  \setlength\tabcolsep{4pt}
  \begin{tabular}{lccccccccccc}
\hline & MME & Textvqa & MMB & MMVet & MMMU & MathV & OCRB & MMStar & Avg & Latency \\
\hline
SV-LLaVA (w/o SK)& 1519 & 56.7 & 72.5 & 31.4 & 38.8 & 21.2 & 30.9 & 36.5 & 41.1 & 103  \\
\multirow{2}{*}{SV-LLaVA} & 1519 & 56.5 & 72.5 & 32.4 & 38.8 & 21.2 & 30.7 & 36.5 & 41.3 & 90  \\
& \textcolor{red}{$+0\%$} & \textcolor{blue}{$-0.4\%$} & \textcolor{red}{$+0\%$} & \textcolor{red}{$+3.2\%$} & \textcolor{red}{$+0\%$} & \textcolor{red}{$+0\%$} & \textcolor{blue}{$-0.6\%$} & \textcolor{red}{$+0\%$} & \textcolor{red}{$+0.5\%$} & \textcolor{red}{$-9.7\%$} \\
\hline
SV-LLaVA-HD (w/o SK)& 1533 & 61.6 & 73.0 & 35.8 & 39.9 & 21.2 & 34.5 & 39.7 & 43.7 & 310  \\
\multirow{2}{*}{SV-LLaVA-HD} & 1534 & 61.2 & 73.1 & 35.9 & 39.7 & 21.2 & 34.1 & 39.3 & 43.5 & 250  \\
& \textcolor{red}{$+0\%$} & \textcolor{blue}{$-0.7\%$} & \textcolor{red}{$+0.1\%$} & \textcolor{red}{$+0.3\%$} & \textcolor{blue}{$-0.5\%$} & \textcolor{red}{$+0\%$} & \textcolor{blue}{$-1.2\%$} & \textcolor{blue}{$-1\%$} & \textcolor{blue}{$-0.5\%$} & \textcolor{red}{$-19.4\%$} \\
\hline
SV-LLaVA-CoS (w/o SK)& 1562 & 63.9 & 76.5 & 40.2 & 40.2 & 21.2 & 37.2 & 41.9 & 45.9 & 126  \\
\multirow{2}{*}{SV-LLaVA-CoS} & 1563 & 63.7 & 76.5 & 41.7 & 40.3 & 21.2 & 37 & 41.9 & 46.0 & 90  \\
& \textcolor{red}{$+0\%$} & \textcolor{blue}{$-0.3\%$} & \textcolor{red}{$+0\%$} & \textcolor{red}{$+3.7\%$} & \textcolor{red}{$+0.2\%$} & \textcolor{red}{$+0\%$} & \textcolor{blue}{$-0.5\%$} & \textcolor{red}{$+0\%$} & \textcolor{red}{$+0.2\%$} & \textcolor{red}{$-28.6\%$} \\
\hline
\end{tabular}
  \caption{\textbf{Ablation on skip KV-cache.} We evaluate SK (using skip KV-cache during inference) of Skip-Vision on different settings.}
  \label{tab:sk}
  \vspace{-3mm}
\end{table*}

\begin{figure*}[htp]
    \centering
    \subfloat[Ablation on LLaVA]{%
        \includegraphics[width=0.3\textwidth]{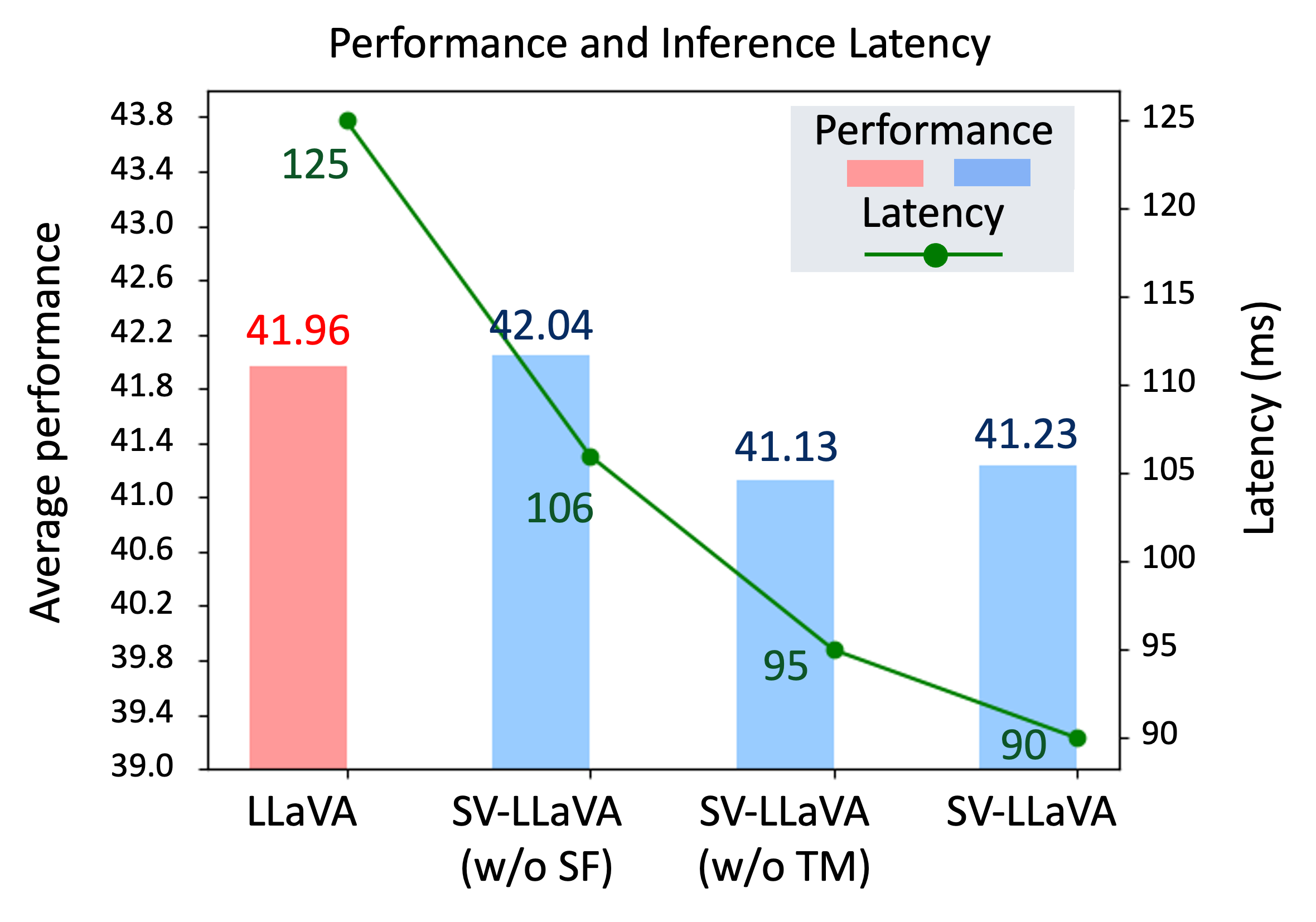}
    }
    \subfloat[Ablation on CoS]{%
        \includegraphics[width=0.3\textwidth]{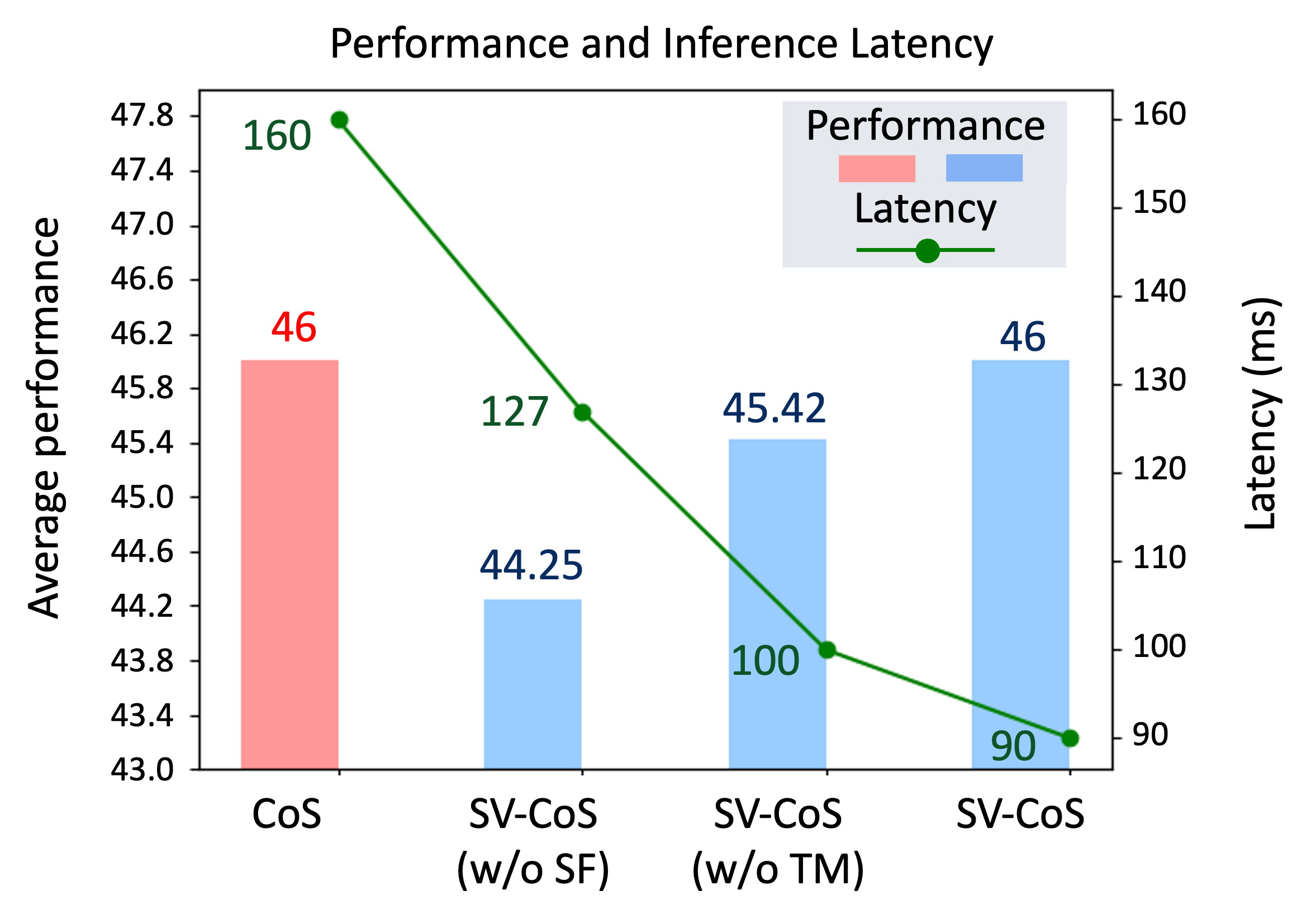}
    }
    \subfloat[Ablation on LLaVA-HD]{%
        \includegraphics[width=0.3\textwidth]{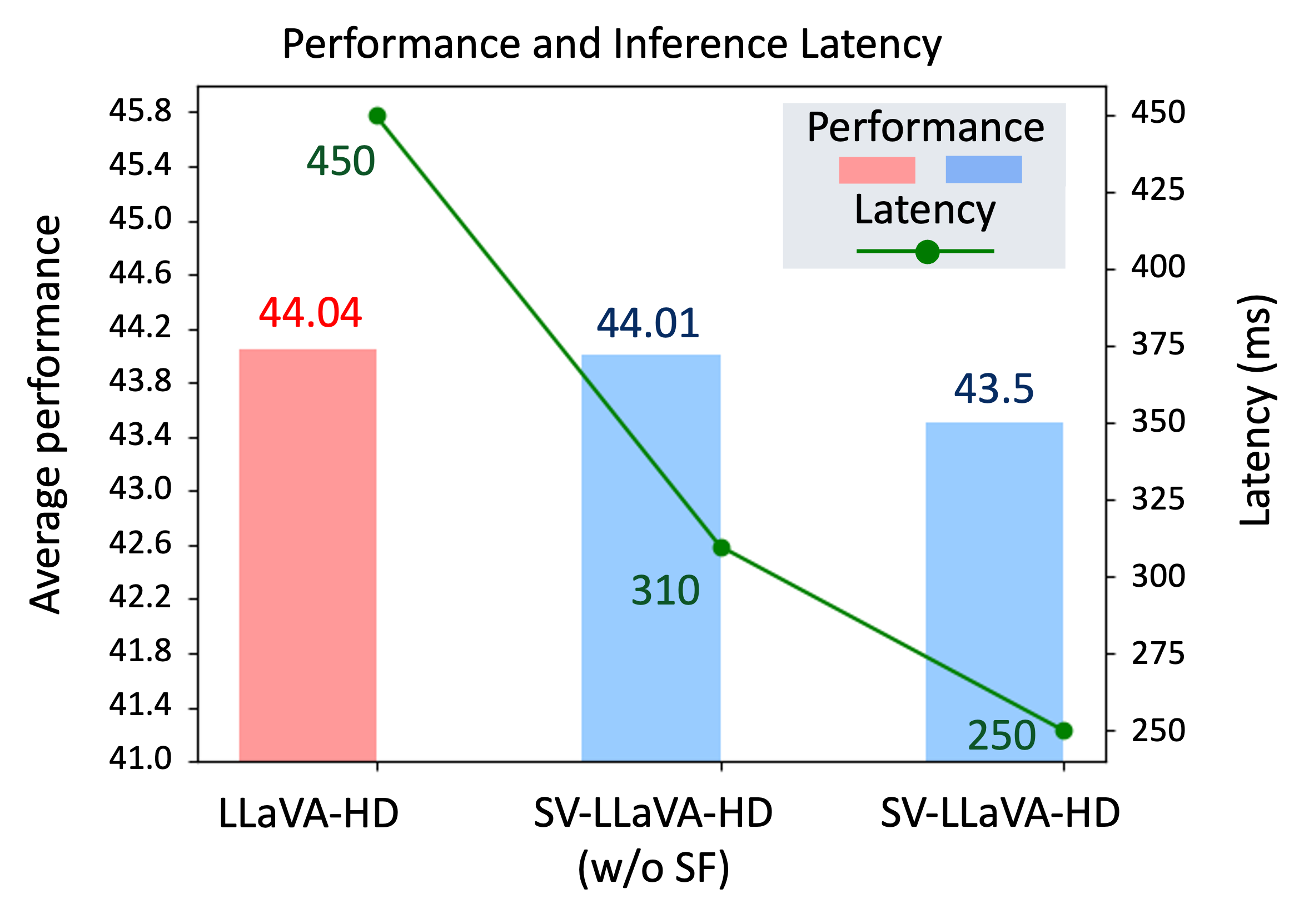}
    }
    \caption{
    \textbf{Ablation study.} We performed an ablation study on each component of the Skip-Vision. SF (skip FFN), TM (token merge). 
    }
    \label{fig:latency}
    \vspace{-5mm}
\end{figure*}
\begin{figure}[htp]
  \centering
   \includegraphics[width=\linewidth]{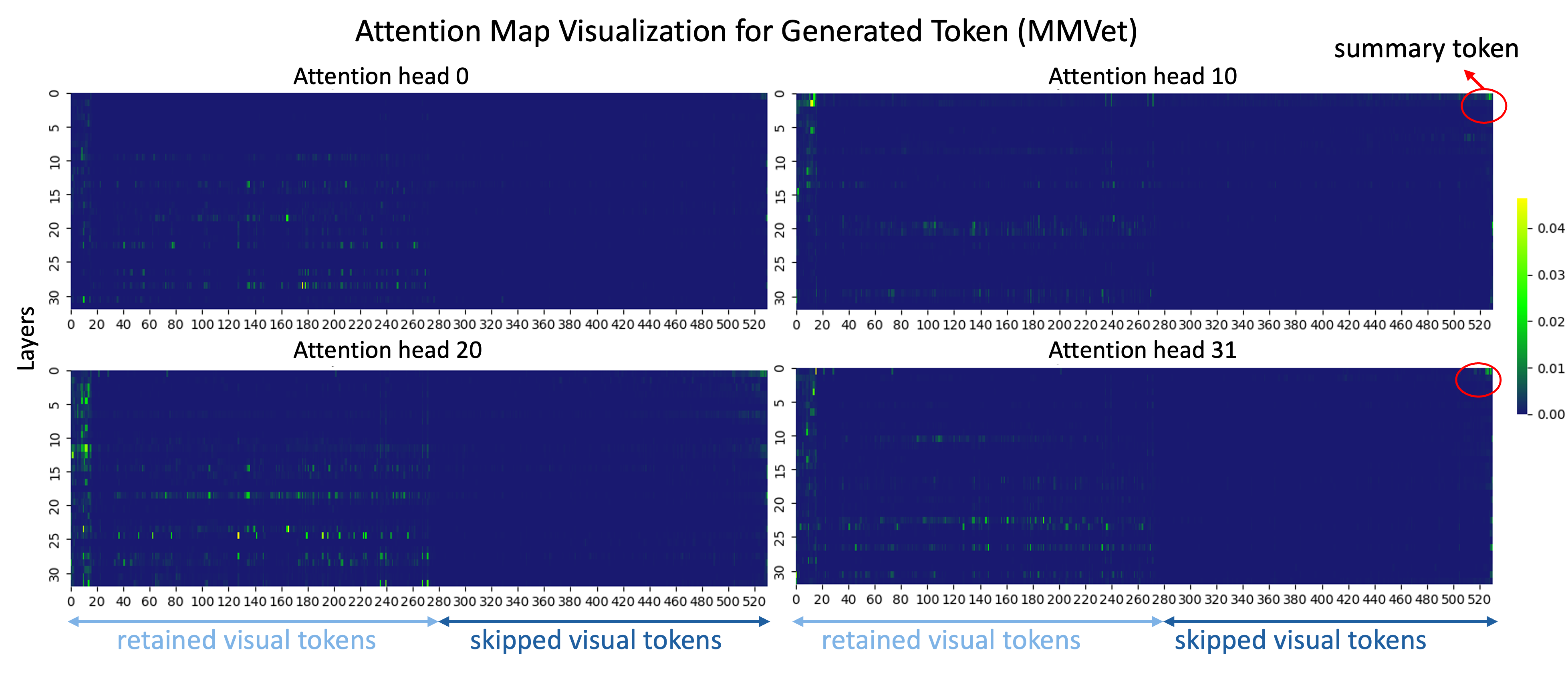}
   \caption{\textbf{Visualization of attention map in MMVet.}}
   \label{fig:6}
   \vspace{-7mm}
\end{figure}

\begin{figure}[htp]
  \centering
   \includegraphics[width=\linewidth]{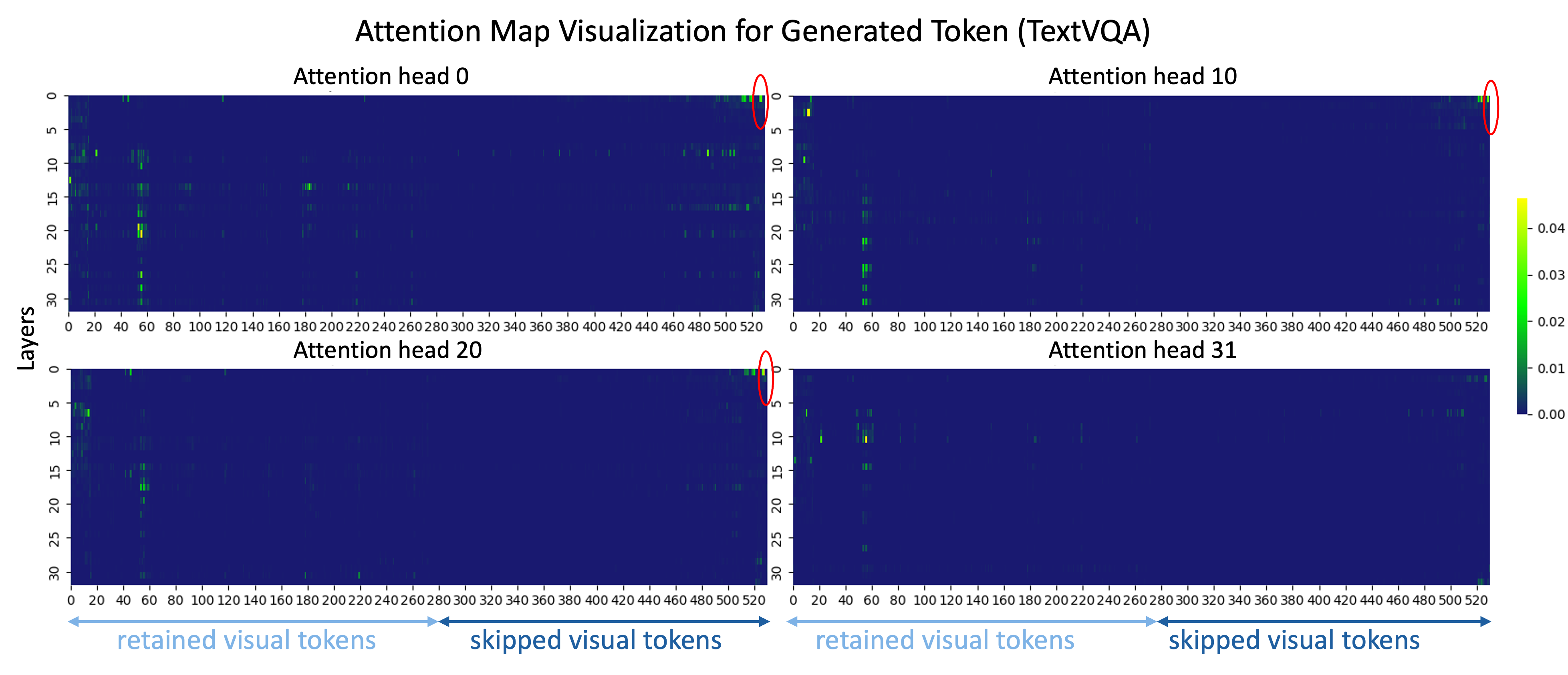}
   \caption{\textbf{Visualization of attention map in TextVQA.}}
   \label{fig:7}
   \vspace{-3mm}
\end{figure}
For LLaVA and LLaVA-HD setting, we use input resolutions of 336 and a dynamic resolution ranging from 336 to 672, respectively, with CLIP ViT-L/336px~\cite{clip} as the vision encoder and LLaMA3 8B Instruct~\cite{llama3} as the backbone (additional comparisons under LLaVA-1.5-7B training setup are in the Appendix \ref{llava-1.5-exp}).
For CoS setting, pretraining is conducted at 224 resolution, training only the multi-scale visual resampler with CLIP ViT-L/224px~\cite{clip}. The model processes 80 visual tokens, including 16 global and 64 local tokens, and scales to 336 and 1296 tokens during SFT. In SFT, the resolution increases to 448, and all model parameters are unfrozen for full optimization.

In all experiments, we used a learning rate of $1e^{-3}$ for pre-training, which is reduced to $1e^{-5}$ during SFT. Training follows a cosine decay schedule after a $3\%$ warm-up phase. For skipped token selection, unless otherwise specified, the LLaVA setting includes 156 skipped tokens and 100 retained tokens, the LLaVA-HD setting has 576 skipped and 576 retained tokens, and the CoS setting consists of 256 skipped and 272 retained tokens. All skipped tokens undergo a token merging process to enhance efficiency.
\subsection{Main experimental results}
\textbf{Performance and efficiency.} 
We first evaluate the effectiveness of Skip-Vision across three different frameworks: CoS, LLaVA, and LLaVA-HD, demonstrating its ability to maximize efficiency while minimizing performance degradation. As shown in Figure \ref{fig:2} and Table \ref{tab:3}, Skip-Vision reduces training time by $22.5\%$, inference FLOPs by $75\%$, and inference latency by $28\%$ for LLaVA. For LLaVA-HD, it achieves a $35.8\%$ reduction in training time, $50\%$ in inference FLOPs, and $45\%$ in inference latency. In CoS, it reduces training time by $35.8\%$, inference FLOPs by $75\%$, and inference latency by $44\%$.

\textbf{Scaling.} We evaluated the scalability of Skip-Vision, leveraging CoS's visual scaling to increase visual tokens from 336 to 1296 while maintaining efficiency. Expanding the fine-tuning dataset from LLaVA-665k to SV-1M further improved performance by 8\% over $CoS_{1296}$ with similar resources. To benchmark against state-of-the-art models, we scaled to SV-9M, with training completed in 72 hours on 16 NVIDIA A100 GPUs (Table \ref{tab:2}).

\subsection{Ablation study and analysis}
To validate the effectiveness of each component within the Skip-Vision framework, we conducted ablation and comparative experiments on token merge, skip FFN, and skip KV-cache.

As shown in Figure \ref{fig:latency}, both the skip FFN strategy and token merging enhance computational efficiency across different model settings while maintaining performance. Notably, skip FFN performs better in the CoS setting than in LLaVA and LLaVA-HD. We attribute this to CoS allowing the vision backbone to be fine-tuned during the SFT stage, enabling key information to be more effectively captured by summary tokens, thereby reducing the errors introduced by skip FFN. As shown in Table \ref{tab:sk}, skip KV-cache slightly affects fine-grained tasks such as TextVQA and OCRBench but improves performance in long-form response tasks like MMVet.

To further analyze this, we visualized attention maps across layers for visual tokens in SV-CoS during the MMVet and TextVQA tasks. In MMVet (Figure \ref{fig:6}), the attention of generated tokens primarily focuses on retained and summary tokens, with skip KV-cache filtering out unnecessary noise. In TextVQA (Figure \ref{fig:7}), which requires fine-grained understanding, attention is spread across both retained and skipped tokens, leading to some performance loss with skip KV-cache. However, causal attention effectively aggregates information from skipped tokens into the final summary tokens, allowing us to retain only the summary tokens and retained visual tokens in the KV-cache.

\section{Conclusion}
We propose Skip-Vision, an efficient multimodal large language model architecture that boosts computational efficiency through a skip-FFN strategy, reducing redundant computations over visual tokens during training. During inference, Skip-Vision uses a skip KV-cache to accelerate processing by omitting non-essential visual tokens from the KV cache. Extensive experiments show our model outperforms current methods in efficiency and scales effectively with data, delivering competitive performance among leading models of similar scale.

\section{Acknowledgements}
This work was supported in part by NSFC (62201342), and Shanghai Municipal Science and Technology Major Project (2021SHZDZX0102). This work was supported by Ant Group Research Intern Program.

{
    \small
    \bibliographystyle{ieeenat_fullname}
    \bibliography{main}

\begin{thebibliography}{136}
\providecommand{\natexlab}[1]{#1}
\providecommand{\url}[1]{\texttt{#1}}
\expandafter\ifx\csname urlstyle\endcsname\relax
  \providecommand{\doi}[1]{doi: #1}\else
  \providecommand{\doi}{doi: \begingroup \urlstyle{rm}\Url}\fi

\bibitem[Acharya et~al.(2019)Acharya, Kafle, and Kanan]{tallyqa}
Manoj Acharya, Kushal Kafle, and Christopher Kanan.
\newblock Tallyqa: Answering complex counting questions.
\newblock In \emph{Proceedings of the AAAI conference on artificial
  intelligence}, pages 8076--8084, 2019.

\bibitem[Achiam et~al.(2023)Achiam, Adler, Agarwal, Ahmad, Akkaya, Aleman,
  Almeida, Altenschmidt, Altman, Anadkat, et~al.]{gpt4}
Josh Achiam, Steven Adler, Sandhini Agarwal, Lama Ahmad, Ilge Akkaya,
  Florencia~Leoni Aleman, Diogo Almeida, Janko Altenschmidt, Sam Altman,
  Shyamal Anadkat, et~al.
\newblock Gpt-4 technical report.
\newblock \emph{arXiv preprint arXiv:2303.08774}, 2023.

\bibitem[Alayrac et~al.(2022)Alayrac, Donahue, Luc, Miech, Barr, Hasson, Lenc,
  Mensch, Millican, Reynolds, et~al.]{flamingo}
Jean-Baptiste Alayrac, Jeff Donahue, Pauline Luc, Antoine Miech, Iain Barr,
  Yana Hasson, Karel Lenc, Arthur Mensch, Katherine Millican, Malcolm Reynolds,
  et~al.
\newblock Flamingo: a visual language model for few-shot learning.
\newblock \emph{Advances in neural information processing systems},
  35:\penalty0 23716--23736, 2022.

\bibitem[Antol et~al.(2015)Antol, Agrawal, Lu, Mitchell, Batra, Zitnick, and
  Parikh]{vqav2}
Stanislaw Antol, Aishwarya Agrawal, Jiasen Lu, Margaret Mitchell, Dhruv Batra,
  C~Lawrence Zitnick, and Devi Parikh.
\newblock Vqa: Visual question answering.
\newblock In \emph{Proceedings of the IEEE international conference on computer
  vision}, pages 2425--2433, 2015.

\bibitem[Baechler et~al.(2024)Baechler, Sunkara, Wang, Zubach, Mansoor, Etter,
  C{\u{a}}rbune, Lin, Chen, and Sharma]{screenai}
Gilles Baechler, Srinivas Sunkara, Maria Wang, Fedir Zubach, Hassan Mansoor,
  Vincent Etter, Victor C{\u{a}}rbune, Jason Lin, Jindong Chen, and Abhanshu
  Sharma.
\newblock Screenai: A vision-language model for ui and infographics
  understanding.
\newblock \emph{arXiv preprint arXiv:2402.04615}, 2024.

\bibitem[Bai et~al.(2023)Bai, Bai, Yang, Wang, Tan, Wang, Lin, Zhou, and
  Zhou]{qwenvl}
Jinze Bai, Shuai Bai, Shusheng Yang, Shijie Wang, Sinan Tan, Peng Wang, Junyang
  Lin, Chang Zhou, and Jingren Zhou.
\newblock Qwen-vl: A versatile vision-language model for understanding,
  localization, text reading, and beyond.
\newblock \emph{arXiv preprint arXiv:2308.12966}, 1\penalty0 (2):\penalty0 3,
  2023.

\bibitem[Basu et~al.(2024)Basu, Grayson, Morrison, Nushi, Feizi, and
  Massiceti]{understanding-mllm}
Samyadeep Basu, Martin Grayson, Cecily Morrison, Besmira Nushi, Soheil Feizi,
  and Daniela Massiceti.
\newblock Understanding information storage and transfer in multi-modal large
  language models.
\newblock \emph{arXiv preprint arXiv:2406.04236}, 2024.

\bibitem[Biten et~al.(2019)Biten, Tito, Mafla, Gomez, Rusinol, Valveny,
  Jawahar, and Karatzas]{stvqa}
Ali~Furkan Biten, Ruben Tito, Andres Mafla, Lluis Gomez, Mar{\c{c}}al Rusinol,
  Ernest Valveny, CV Jawahar, and Dimosthenis Karatzas.
\newblock Scene text visual question answering.
\newblock In \emph{Proceedings of the IEEE/CVF international conference on
  computer vision}, pages 4291--4301, 2019.

\bibitem[Bolya et~al.(2022)Bolya, Fu, Dai, Zhang, Feichtenhofer, and
  Hoffman]{tome}
Daniel Bolya, Cheng-Yang Fu, Xiaoliang Dai, Peizhao Zhang, Christoph
  Feichtenhofer, and Judy Hoffman.
\newblock Token merging: Your vit but faster.
\newblock \emph{arXiv preprint arXiv:2210.09461}, 2022.

\bibitem[Brown(2020)]{gpt}
Tom~B Brown.
\newblock Language models are few-shot learners.
\newblock \emph{arXiv preprint arXiv:2005.14165}, 2020.

\bibitem[Cai et~al.(2024)Cai, Bao, Guo, Zhang, Song, and Zheng]{geogpt4v}
Shihao Cai, Keqin Bao, Hangyu Guo, Jizhi Zhang, Jun Song, and Bo Zheng.
\newblock Geogpt4v: Towards geometric multi-modal large language models with
  geometric image generation.
\newblock \emph{arXiv preprint arXiv:2406.11503}, 2024.

\bibitem[Cao et~al.(2024)Cao, Ye, Li, Yu, Tang, Lu, and Chen]{madtp}
Jianjian Cao, Peng Ye, Shengze Li, Chong Yu, Yansong Tang, Jiwen Lu, and Tao
  Chen.
\newblock Madtp: Multimodal alignment-guided dynamic token pruning for
  accelerating vision-language transformer.
\newblock In \emph{Proceedings of the IEEE/CVF Conference on Computer Vision
  and Pattern Recognition}, pages 15710--15719, 2024.

\bibitem[Carter(2024)]{textocr-gpt4v}
Jimmy Carter.
\newblock Textocr-gpt4v.
\newblock \url{https://huggingface.co/datasets/jimmycarter/textocr-gpt4v},
  2024.

\bibitem[Cha et~al.(2024)Cha, Kang, Mun, and Roh]{honeybee}
Junbum Cha, Wooyoung Kang, Jonghwan Mun, and Byungseok Roh.
\newblock Honeybee: Locality-enhanced projector for multimodal llm.
\newblock In \emph{Proceedings of the IEEE/CVF Conference on Computer Vision
  and Pattern Recognition}, pages 13817--13827, 2024.

\bibitem[Chai et~al.(2024)Chai, Song, Du, Meng, Madhavan, Bar-Tal, Hwang, Xie,
  and Manning]{auroracap}
Wenhao Chai, Enxin Song, Yilun Du, Chenlin Meng, Vashisht Madhavan, Omer
  Bar-Tal, Jenq-Neng Hwang, Saining Xie, and Christopher~D Manning.
\newblock Auroracap: Efficient, performant video detailed captioning and a new
  benchmark.
\newblock \emph{arXiv preprint arXiv:2410.03051}, 2024.

\bibitem[Chan(2020)]{mathocr}
Chungkwong Chan.
\newblock Stroke extraction for offline handwritten mathematical expression
  recognition.
\newblock \emph{IEEE Access}, 8:\penalty0 61565--61575, 2020.

\bibitem[Chang et~al.(2022)Chang, Palzer, Li, Fosler-Lussier, and Xiao]{mapqa}
Shuaichen Chang, David Palzer, Jialin Li, Eric Fosler-Lussier, and Ningchuan
  Xiao.
\newblock Mapqa: A dataset for question answering on choropleth maps.
\newblock \emph{arXiv preprint arXiv:2211.08545}, 2022.

\bibitem[Chen et~al.(2024{\natexlab{a}})Chen, Chen, Zhang, Chen, Wu, Zhang,
  Chen, Li, Wan, and Wang]{allava}
Guiming~Hardy Chen, Shunian Chen, Ruifei Zhang, Junying Chen, Xiangbo Wu, Zhiyi
  Zhang, Zhihong Chen, Jianquan Li, Xiang Wan, and Benyou Wang.
\newblock Allava: Harnessing gpt4v-synthesized data for a lite vision-language
  model, 2024{\natexlab{a}}.

\bibitem[Chen et~al.(2023{\natexlab{a}})Chen, Zhang, Zeng, Zhang, Zhu, and
  Zhao]{shikra}
Keqin Chen, Zhao Zhang, Weili Zeng, Richong Zhang, Feng Zhu, and Rui Zhao.
\newblock Shikra: Unleashing multimodal llm's referential dialogue magic.
\newblock \emph{arXiv preprint arXiv:2306.15195}, 2023{\natexlab{a}}.

\bibitem[Chen et~al.(2023{\natexlab{b}})Chen, Li, Dong, Zhang, He, Wang, Zhao,
  and Lin]{sharegpt4v}
Lin Chen, Jisong Li, Xiaoyi Dong, Pan Zhang, Conghui He, Jiaqi Wang, Feng Zhao,
  and Dahua Lin.
\newblock Sharegpt4v: Improving large multi-modal models with better captions.
\newblock \emph{arXiv preprint arXiv:2311.12793}, 2023{\natexlab{b}}.

\bibitem[Chen et~al.(2024{\natexlab{b}})Chen, Li, Dong, Zhang, Zang, Chen,
  Duan, Wang, Qiao, Lin, et~al.]{mmstar}
Lin Chen, Jinsong Li, Xiaoyi Dong, Pan Zhang, Yuhang Zang, Zehui Chen, Haodong
  Duan, Jiaqi Wang, Yu Qiao, Dahua Lin, et~al.
\newblock Are we on the right way for evaluating large vision-language models?
\newblock \emph{arXiv preprint arXiv:2403.20330}, 2024{\natexlab{b}}.

\bibitem[Chen et~al.(2024{\natexlab{c}})Chen, Zhao, Liu, Bai, Lin, Zhou, and
  Chang]{fastv}
Liang Chen, Haozhe Zhao, Tianyu Liu, Shuai Bai, Junyang Lin, Chang Zhou, and
  Baobao Chang.
\newblock An image is worth 1/2 tokens after layer 2: Plug-and-play inference
  acceleration for large vision-language models.
\newblock \emph{arXiv preprint arXiv:2403.06764}, 2024{\natexlab{c}}.

\bibitem[Chen et~al.(2021)Chen, Zhao, Chen, Zhang, Ji, Luo, Xiong, and
  Yu]{websrc}
Xingyu Chen, Zihan Zhao, Lu Chen, Danyang Zhang, Jiabao Ji, Ao Luo, Yuxuan
  Xiong, and Kai Yu.
\newblock Websrc: A dataset for web-based structural reading comprehension.
\newblock \emph{arXiv preprint arXiv:2101.09465}, 2021.

\bibitem[Chen et~al.(2024{\natexlab{d}})Chen, Wang, Tian, Ye, Gao, Cui, Tong,
  Hu, Luo, Ma, et~al.]{InternVL}
Zhe Chen, Weiyun Wang, Hao Tian, Shenglong Ye, Zhangwei Gao, Erfei Cui, Wenwen
  Tong, Kongzhi Hu, Jiapeng Luo, Zheng Ma, et~al.
\newblock How far are we to gpt-4v? closing the gap to commercial multimodal
  models with open-source suites.
\newblock \emph{arXiv preprint arXiv:2404.16821}, 2024{\natexlab{d}}.

\bibitem[Chiang et~al.(2023)Chiang, Li, Lin, Sheng, Wu, Zhang, Zheng, Zhuang,
  Zhuang, Gonzalez, Stoica, and Xing]{vicuna}
Wei-Lin Chiang, Zhuohan Li, Zi Lin, Ying Sheng, Zhanghao Wu, Hao Zhang, Lianmin
  Zheng, Siyuan Zhuang, Yonghao Zhuang, Joseph~E. Gonzalez, Ion Stoica, and
  Eric~P. Xing.
\newblock Vicuna: An open-source chatbot impressing gpt-4 with 90\%* chatgpt
  quality.
\newblock \url{https://lmsys.org/blog/2023-03-30-vicuna/}, 2023.

\bibitem[Chng et~al.(2019)Chng, Liu, Sun, Ng, Luo, Ni, Fang, Zhang, Han, Ding,
  et~al.]{icdar-art}
Chee~Kheng Chng, Yuliang Liu, Yipeng Sun, Chun~Chet Ng, Canjie Luo, Zihan Ni,
  ChuanMing Fang, Shuaitao Zhang, Junyu Han, Errui Ding, et~al.
\newblock Icdar2019 robust reading challenge on arbitrary-shaped text-rrc-art.
\newblock In \emph{2019 International Conference on Document Analysis and
  Recognition (ICDAR)}, pages 1571--1576. IEEE, 2019.

\bibitem[Chu et~al.(2023)Chu, Qiao, Lin, Xu, Yang, Hu, Wei, Zhang, Zhang, Wei,
  et~al.]{mobilevlm}
Xiangxiang Chu, Limeng Qiao, Xinyang Lin, Shuang Xu, Yang Yang, Yiming Hu, Fei
  Wei, Xinyu Zhang, Bo Zhang, Xiaolin Wei, et~al.
\newblock Mobilevlm: A fast, reproducible and strong vision language assistant
  for mobile devices.
\newblock \emph{arXiv preprint arXiv:2312.16886}, 2023.

\bibitem[Dai et~al.(2023)Dai, Li, Li, Tiong, Zhao, Wang, Li, Fung, and
  Hoi]{instructblip}
Wenliang Dai, Junnan Li, Dongxu Li, Anthony Tiong, Junqi Zhao, Weisheng Wang,
  Boyang Li, Pascale Fung, and Steven Hoi.
\newblock Instruct{BLIP}: Towards general-purpose vision-language models with
  instruction tuning.
\newblock In \emph{Thirty-seventh Conference on Neural Information Processing
  Systems}, 2023.

\bibitem[Dubey et~al.(2024)Dubey, Jauhri, Pandey, Kadian, Al-Dahle, Letman,
  Mathur, Schelten, Yang, Fan, et~al.]{llama3}
Abhimanyu Dubey, Abhinav Jauhri, Abhinav Pandey, Abhishek Kadian, Ahmad
  Al-Dahle, Aiesha Letman, Akhil Mathur, Alan Schelten, Amy Yang, Angela Fan,
  et~al.
\newblock The llama 3 herd of models.
\newblock \emph{arXiv preprint arXiv:2407.21783}, 2024.

\bibitem[Fu et~al.(2024)Fu, Chen, Shen, Qin, Zhang, Lin, Yang, Zheng, Li, Sun,
  Wu, and Ji]{mme}
Chaoyou Fu, Peixian Chen, Yunhang Shen, Yulei Qin, Mengdan Zhang, Xu Lin,
  Jinrui Yang, Xiawu Zheng, Ke Li, Xing Sun, Yunsheng Wu, and Rongrong Ji.
\newblock Mme: A comprehensive evaluation benchmark for multimodal large
  language models.
\newblock \emph{arXiv preprint arXiv:2306.13394}, 2024.

\bibitem[Gao et~al.(2023)Gao, Pi, Zhang, Ye, Zhong, Wang, Hong, Han, Xu, Li,
  et~al.]{geo170k}
Jiahui Gao, Renjie Pi, Jipeng Zhang, Jiacheng Ye, Wanjun Zhong, Yufei Wang,
  Lanqing Hong, Jianhua Han, Hang Xu, Zhenguo Li, et~al.
\newblock G-llava: Solving geometric problem with multi-modal large language
  model.
\newblock \emph{arXiv preprint arXiv:2312.11370}, 2023.

\bibitem[Geva et~al.(2023)Geva, Bastings, Filippova, and Globerson]{dissecting}
Mor Geva, Jasmijn Bastings, Katja Filippova, and Amir Globerson.
\newblock Dissecting recall of factual associations in auto-regressive language
  models.
\newblock \emph{arXiv preprint arXiv:2304.14767}, 2023.

\bibitem[Gupta et~al.(2016)Gupta, Vedaldi, and Zisserman]{synthetic}
Ankush Gupta, Andrea Vedaldi, and Andrew Zisserman.
\newblock Synthetic data for text localisation in natural images.
\newblock In \emph{Proceedings of the IEEE conference on computer vision and
  pattern recognition}, pages 2315--2324, 2016.

\bibitem[Gurari et~al.(2019)Gurari, Li, Lin, Zhao, Guo, Stangl, and
  Bigham]{vizwiz}
Danna Gurari, Qing Li, Chi Lin, Yinan Zhao, Anhong Guo, Abigale Stangl, and
  Jeffrey~P Bigham.
\newblock Vizwiz-priv: A dataset for recognizing the presence and purpose of
  private visual information in images taken by blind people.
\newblock In \emph{Proceedings of the IEEE/CVF Conference on Computer Vision
  and Pattern Recognition}, pages 939--948, 2019.

\bibitem[Han et~al.(2024)Han, Wang, Peng, Xiong, Chen, Ji, and
  Wang]{lminfinite}
Chi Han, Qifan Wang, Hao Peng, Wenhan Xiong, Yu Chen, Heng Ji, and Sinong Wang.
\newblock Lm-infinite: Zero-shot extreme length generalization for large
  language models.
\newblock In \emph{Proceedings of the 2024 Conference of the North American
  Chapter of the Association for Computational Linguistics: Human Language
  Technologies (Volume 1: Long Papers)}, pages 3991--4008, 2024.

\bibitem[He et~al.(2018)He, Liu, Yang, Zhang, Luo, Gao, Zheng, Wang, Zhang, and
  Jin]{icpr}
Mengchao He, Yuliang Liu, Zhibo Yang, Sheng Zhang, Canjie Luo, Feiyu Gao, Qi
  Zheng, Yongpan Wang, Xin Zhang, and Lianwen Jin.
\newblock Icpr2018 contest on robust reading for multi-type web images.
\newblock In \emph{2018 24th international conference on pattern recognition
  (ICPR)}, pages 7--12. IEEE, 2018.

\bibitem[He et~al.(2024)He, Liu, Wu, Yuan, Wang, Huang, and Zhao]{bunny}
Muyang He, Yexin Liu, Boya Wu, Jianhao Yuan, Yueze Wang, Tiejun Huang, and Bo
  Zhao.
\newblock Efficient multimodal learning from data-centric perspective.
\newblock \emph{arXiv preprint arXiv:2402.11530}, 2024.

\bibitem[Hu et~al.(2024{\natexlab{a}})Hu, Xu, Ye, Yan, Zhang, Zhang, Li, Zhang,
  Jin, Huang, et~al.]{docreason}
Anwen Hu, Haiyang Xu, Jiabo Ye, Ming Yan, Liang Zhang, Bo Zhang, Chen Li, Ji
  Zhang, Qin Jin, Fei Huang, et~al.
\newblock mplug-docowl 1.5: Unified structure learning for ocr-free document
  understanding.
\newblock \emph{arXiv preprint arXiv:2403.12895}, 2024{\natexlab{a}}.

\bibitem[Hu et~al.(2024{\natexlab{b}})Hu, Xu, Zhang, Ye, Yan, Zhang, Jin,
  Huang, and Zhou]{DocStruct4M}
Anwen Hu, Haiyang Xu, Liang Zhang, Jiabo Ye, Ming Yan, Ji Zhang, Qin Jin, Fei
  Huang, and Jingren Zhou.
\newblock mplug-docowl2: High-resolution compressing for ocr-free multi-page
  document understanding.
\newblock \emph{arXiv preprint arXiv:2409.03420}, 2024{\natexlab{b}}.

\bibitem[Hu et~al.(2024{\natexlab{c}})Hu, Dou, Li, Kamath, Peng, and
  Chang]{mqt}
Wenbo Hu, Zi-Yi Dou, Liunian~Harold Li, Amita Kamath, Nanyun Peng, and Kai-Wei
  Chang.
\newblock Matryoshka query transformer for large vision-language models.
\newblock \emph{arXiv preprint arXiv:2405.19315}, 2024{\natexlab{c}}.

\bibitem[Huang et~al.(2024)Huang, Ji, Gong, Qing, Zhang, Zheng, Wang, Chen, and
  Yang]{cos}
Ziyuan Huang, Kaixiang Ji, Biao Gong, Zhiwu Qing, Qinglong Zhang, Kecheng
  Zheng, Jian Wang, Jingdong Chen, and Ming Yang.
\newblock Accelerating pre-training of multimodal llms via chain-of-sight.
\newblock \emph{arXiv preprint arXiv:2407.15819}, 2024.

\bibitem[Jaegle et~al.(2021)Jaegle, Gimeno, Brock, Vinyals, Zisserman, and
  Carreira]{perceiver}
Andrew Jaegle, Felix Gimeno, Andy Brock, Oriol Vinyals, Andrew Zisserman, and
  Joao Carreira.
\newblock Perceiver: General perception with iterative attention.
\newblock In \emph{International conference on machine learning}, pages
  4651--4664. PMLR, 2021.

\bibitem[Jin et~al.(2024)Jin, Li, Liu, Gu, Wu, Jiang, He, Zhao, Tan, Gan,
  et~al.]{efficient}
Yizhang Jin, Jian Li, Yexin Liu, Tianjun Gu, Kai Wu, Zhengkai Jiang, Muyang He,
  Bo Zhao, Xin Tan, Zhenye Gan, et~al.
\newblock Efficient multimodal large language models: A survey.
\newblock \emph{arXiv preprint arXiv:2405.10739}, 2024.

\bibitem[Johnson et~al.(2017)Johnson, Hariharan, Van Der~Maaten, Fei-Fei,
  Lawrence~Zitnick, and Girshick]{clevr}
Justin Johnson, Bharath Hariharan, Laurens Van Der~Maaten, Li Fei-Fei, C
  Lawrence~Zitnick, and Ross Girshick.
\newblock Clevr: A diagnostic dataset for compositional language and elementary
  visual reasoning.
\newblock In \emph{Proceedings of the IEEE conference on computer vision and
  pattern recognition}, pages 2901--2910, 2017.

\bibitem[Kafle et~al.(2018)Kafle, Price, Cohen, and Kanan]{dvqa}
Kushal Kafle, Brian Price, Scott Cohen, and Christopher Kanan.
\newblock Dvqa: Understanding data visualizations via question answering.
\newblock In \emph{Proceedings of the IEEE conference on computer vision and
  pattern recognition}, pages 5648--5656, 2018.

\bibitem[Kahou et~al.(2017)Kahou, Michalski, Atkinson, K{\'a}d{\'a}r,
  Trischler, and Bengio]{figureqa}
Samira~Ebrahimi Kahou, Vincent Michalski, Adam Atkinson, {\'A}kos
  K{\'a}d{\'a}r, Adam Trischler, and Yoshua Bengio.
\newblock Figureqa: An annotated figure dataset for visual reasoning.
\newblock \emph{arXiv preprint arXiv:1710.07300}, 2017.

\bibitem[Kantharaj et~al.(2022)Kantharaj, Leong, Lin, Masry, Thakkar, Hoque,
  and Joty]{charttext}
Shankar Kantharaj, Rixie Tiffany~Ko Leong, Xiang Lin, Ahmed Masry, Megh
  Thakkar, Enamul Hoque, and Shafiq Joty.
\newblock Chart-to-text: A large-scale benchmark for chart summarization.
\newblock \emph{arXiv preprint arXiv:2203.06486}, 2022.

\bibitem[Kar et~al.(2024)Kar, Tonioni, Poklukar, Kulshrestha, Zamir, and
  Tombari]{brave}
O{\u{g}}uzhan~Fatih Kar, Alessio Tonioni, Petra Poklukar, Achin Kulshrestha,
  Amir Zamir, and Federico Tombari.
\newblock Brave: Broadening the visual encoding of vision-language models.
\newblock \emph{arXiv preprint arXiv:2404.07204}, 2024.

\bibitem[Karatzas et~al.(2015)Karatzas, Gomez-Bigorda, Nicolaou, Ghosh,
  Bagdanov, Iwamura, Matas, Neumann, Chandrasekhar, Lu, et~al.]{icdar}
Dimosthenis Karatzas, Lluis Gomez-Bigorda, Anguelos Nicolaou, Suman Ghosh,
  Andrew Bagdanov, Masakazu Iwamura, Jiri Matas, Lukas Neumann,
  Vijay~Ramaseshan Chandrasekhar, Shijian Lu, et~al.
\newblock Icdar 2015 competition on robust reading.
\newblock In \emph{2015 13th international conference on document analysis and
  recognition (ICDAR)}, pages 1156--1160. IEEE, 2015.

\bibitem[Kembhavi et~al.(2016)Kembhavi, Salvato, Kolve, Seo, Hajishirzi, and
  Farhadi]{ai2d}
Aniruddha Kembhavi, Mike Salvato, Eric Kolve, Minjoon Seo, Hannaneh Hajishirzi,
  and Ali Farhadi.
\newblock A diagram is worth a dozen images.
\newblock In \emph{Computer Vision--ECCV 2016: 14th European Conference,
  Amsterdam, The Netherlands, October 11--14, 2016, Proceedings, Part IV 14},
  pages 235--251. Springer, 2016.

\bibitem[Kembhavi et~al.(2017)Kembhavi, Seo, Schwenk, Choi, Farhadi, and
  Hajishirzi]{tqa}
Aniruddha Kembhavi, Minjoon Seo, Dustin Schwenk, Jonghyun Choi, Ali Farhadi,
  and Hannaneh Hajishirzi.
\newblock Are you smarter than a sixth grader? textbook question answering for
  multimodal machine comprehension.
\newblock In \emph{Proceedings of the IEEE Conference on Computer Vision and
  Pattern recognition}, pages 4999--5007, 2017.

\bibitem[Kweon et~al.(2023)Kweon, Kwon, Cho, Jo, and Choi]{wikitable}
Sunjun Kweon, Yeonsu Kwon, Seonhee Cho, Yohan Jo, and Edward Choi.
\newblock Open-wikitable: Dataset for open domain question answering with
  complex reasoning over table.
\newblock \emph{arXiv preprint arXiv:2305.07288}, 2023.

\bibitem[Lau et~al.(2018)Lau, Gayen, Ben~Abacha, and Demner-Fushman]{vqa-rad}
Jason~J Lau, Soumya Gayen, Asma Ben~Abacha, and Dina Demner-Fushman.
\newblock A dataset of clinically generated visual questions and answers about
  radiology images.
\newblock \emph{Scientific data}, 5\penalty0 (1):\penalty0 1--10, 2018.

\bibitem[Laurençon et~al.(2024)Laurençon, Marafioti, Sanh, and
  Tronchon]{docmatix}
Hugo Laurençon, Andrés Marafioti, Victor Sanh, and Léo Tronchon.
\newblock Building and better understanding vision-language models: insights
  and future directions., 2024.

\bibitem[Lee et~al.(2023)Lee, Park, Jo, and Seo]{volcano}
Seongyun Lee, Sue~Hyun Park, Yongrae Jo, and Minjoon Seo.
\newblock Volcano: mitigating multimodal hallucination through self-feedback
  guided revision.
\newblock \emph{arXiv preprint arXiv:2311.07362}, 2023.

\bibitem[Lerner et~al.(2022)Lerner, Ferret, Guinaudeau, Le~Borgne,
  Besan{\c{c}}on, Moreno, and Lov{\'o}n~Melgarejo]{viquae}
Paul Lerner, Olivier Ferret, Camille Guinaudeau, Herv{\'e} Le~Borgne, Romaric
  Besan{\c{c}}on, Jos{\'e}~G Moreno, and Jes{\'u}s Lov{\'o}n~Melgarejo.
\newblock Viquae, a dataset for knowledge-based visual question answering about
  named entities.
\newblock In \emph{Proceedings of the 45th International ACM SIGIR Conference
  on Research and Development in Information Retrieval}, pages 3108--3120,
  2022.

\bibitem[Li et~al.(2024{\natexlab{a}})Li, Zhang, Zhang, Guo, Zhang, Li, Zhang,
  Liu, and Li]{li2024llavanext-strong}
Bo Li, Kaichen Zhang, Hao Zhang, Dong Guo, Renrui Zhang, Feng Li, Yuanhan
  Zhang, Ziwei Liu, and Chunyuan Li.
\newblock Llava-next: Stronger llms supercharge multimodal capabilities in the
  wild.
\newblock
  \url{https://llava-vl.github.io/blog/2024-05-10-llava-next-stronger-llms/},
  2024{\natexlab{a}}.

\bibitem[Li et~al.(2024{\natexlab{b}})Li, Zhang, Guo, Zhang, Li, Zhang, Zhang,
  Zhang, Li, Liu, et~al.]{k12}
Bo Li, Yuanhan Zhang, Dong Guo, Renrui Zhang, Feng Li, Hao Zhang, Kaichen
  Zhang, Peiyuan Zhang, Yanwei Li, Ziwei Liu, et~al.
\newblock Llava-onevision: Easy visual task transfer.
\newblock \emph{arXiv preprint arXiv:2408.03326}, 2024{\natexlab{b}}.

\bibitem[Li et~al.(2023{\natexlab{a}})Li, Li, Savarese, and Hoi]{blip2}
Junnan Li, Dongxu Li, Silvio Savarese, and Steven Hoi.
\newblock Blip-2: Bootstrapping language-image pre-training with frozen image
  encoders and large language models.
\newblock In \emph{International conference on machine learning}, pages
  19730--19742. PMLR, 2023{\natexlab{a}}.

\bibitem[Li et~al.(2024{\natexlab{c}})Li, Wang, Xu, Wang, Feng, Kong, and
  Liu]{arxivvqa}
Lei Li, Yuqi Wang, Runxin Xu, Peiyi Wang, Xiachong Feng, Lingpeng Kong, and Qi
  Liu.
\newblock Multimodal {A}r{X}iv: A dataset for improving scientific
  comprehension of large vision-language models.
\newblock In \emph{Proceedings of the 62nd Annual Meeting of the Association
  for Computational Linguistics (Volume 1: Long Papers)}, pages 14369--14387,
  Bangkok, Thailand, 2024{\natexlab{c}}. Association for Computational
  Linguistics.

\bibitem[Li et~al.(2024{\natexlab{d}})Li, Yuan, Liu, Tang, Wang, Qin, Zhu, and
  Zhang]{tokenpacker}
Wentong Li, Yuqian Yuan, Jian Liu, Dongqi Tang, Song Wang, Jie Qin, Jianke Zhu,
  and Lei Zhang.
\newblock Tokenpacker: Efficient visual projector for multimodal llm.
\newblock \emph{arXiv preprint arXiv:2407.02392}, 2024{\natexlab{d}}.

\bibitem[Li et~al.(2024{\natexlab{e}})Li, Zhang, Wang, Zhong, Chen, Chu, Liu,
  and Jia]{mini}
Yanwei Li, Yuechen Zhang, Chengyao Wang, Zhisheng Zhong, Yixin Chen, Ruihang
  Chu, Shaoteng Liu, and Jiaya Jia.
\newblock Mini-gemini: Mining the potential of multi-modality vision language
  models.
\newblock \emph{arXiv preprint arXiv:2403.18814}, 2024{\natexlab{e}}.

\bibitem[Li et~al.(2023{\natexlab{b}})Li, Wang, Stengel-Eskin, Kortylewski, Ma,
  Van~Durme, and Yuille]{superclevr}
Zhuowan Li, Xingrui Wang, Elias Stengel-Eskin, Adam Kortylewski, Wufei Ma,
  Benjamin Van~Durme, and Alan~L Yuille.
\newblock Super-clevr: A virtual benchmark to diagnose domain robustness in
  visual reasoning.
\newblock In \emph{Proceedings of the IEEE/CVF Conference on Computer Vision
  and Pattern Recognition}, pages 14963--14973, 2023{\natexlab{b}}.

\bibitem[Lin et~al.(2024)Lin, Yin, Ping, Molchanov, Shoeybi, and
  Han]{lin2024vila}
Ji Lin, Hongxu Yin, Wei Ping, Pavlo Molchanov, Mohammad Shoeybi, and Song Han.
\newblock Vila: On pre-training for visual language models.
\newblock In \emph{Proceedings of the IEEE/CVF Conference on Computer Vision
  and Pattern Recognition}, pages 26689--26699, 2024.

\bibitem[Liu et~al.(2024{\natexlab{a}})Liu, Zhang, Qiu, Huang, Lin, Zhao, Geng,
  Lin, Jin, Zhang, et~al.]{sphinx}
Dongyang Liu, Renrui Zhang, Longtian Qiu, Siyuan Huang, Weifeng Lin, Shitian
  Zhao, Shijie Geng, Ziyi Lin, Peng Jin, Kaipeng Zhang, et~al.
\newblock Sphinx-x: Scaling data and parameters for a family of multi-modal
  large language models.
\newblock \emph{arXiv preprint arXiv:2402.05935}, 2024{\natexlab{a}}.

\bibitem[Liu et~al.(2023{\natexlab{a}})Liu, Emerson, and Collier]{vsr}
Fangyu Liu, Guy Emerson, and Nigel Collier.
\newblock Visual spatial reasoning.
\newblock \emph{Transactions of the Association for Computational Linguistics},
  11:\penalty0 635--651, 2023{\natexlab{a}}.

\bibitem[Liu et~al.(2023{\natexlab{b}})Liu, Wang, Yao, Chen, Song, Cho, Yacoob,
  and Yu]{mmc}
Fuxiao Liu, Xiaoyang Wang, Wenlin Yao, Jianshu Chen, Kaiqiang Song, Sangwoo
  Cho, Yaser Yacoob, and Dong Yu.
\newblock Mmc: Advancing multimodal chart understanding with large-scale
  instruction tuning.
\newblock \emph{arXiv preprint arXiv:2311.10774}, 2023{\natexlab{b}}.

\bibitem[Liu et~al.(2024{\natexlab{b}})Liu, Li, Li, and Lee]{llava-improved}
Haotian Liu, Chunyuan Li, Yuheng Li, and Yong~Jae Lee.
\newblock Improved baselines with visual instruction tuning.
\newblock In \emph{Proceedings of the IEEE/CVF Conference on Computer Vision
  and Pattern Recognition}, pages 26296--26306, 2024{\natexlab{b}}.

\bibitem[Liu et~al.(2024{\natexlab{c}})Liu, Li, Li, Li, Zhang, Shen, and
  Lee]{llavanext}
Haotian Liu, Chunyuan Li, Yuheng Li, Bo Li, Yuanhan Zhang, Sheng Shen, and
  Yong~Jae Lee.
\newblock Llava-next: Improved reasoning, ocr, and world knowledge.
\newblock \url{https://llava-vl.github.io/blog/2024-01-30-llava-next/},
  2024{\natexlab{c}}.

\bibitem[Liu et~al.(2024{\natexlab{d}})Liu, Li, Wu, and Lee]{llava}
Haotian Liu, Chunyuan Li, Qingyang Wu, and Yong~Jae Lee.
\newblock Visual instruction tuning.
\newblock \emph{Advances in neural information processing systems}, 36,
  2024{\natexlab{d}}.

\bibitem[Liu et~al.(2023{\natexlab{c}})Liu, Li, Yang, Li, Yin, Liu, Jin, and
  Bai]{ocr}
Yuliang Liu, Zhang Li, Biao Yang, Chunyuan Li, Xucheng Yin, Cheng-lin Liu,
  Lianwen Jin, and Xiang Bai.
\newblock On the hidden mystery of ocr in large multimodal models.
\newblock \emph{arXiv preprint arXiv:2305.07895}, 2023{\natexlab{c}}.

\bibitem[Liu et~al.(2023{\natexlab{d}})Liu, Li, Yang, Li, Yin, Liu, Jin, and
  Bai]{ocrbench-kv}
Yuliang Liu, Zhang Li, Biao Yang, Chunyuan Li, Xucheng Yin, Cheng-lin Liu,
  Lianwen Jin, and Xiang Bai.
\newblock On the hidden mystery of ocr in large multimodal models.
\newblock \emph{arXiv preprint arXiv:2305.07895}, 2023{\natexlab{d}}.

\bibitem[Liu et~al.(2025)Liu, Duan, Zhang, Li, Zhang, Zhao, Yuan, Wang, He,
  Liu, et~al.]{mmbench}
Yuan Liu, Haodong Duan, Yuanhan Zhang, Bo Li, Songyang Zhang, Wangbo Zhao, Yike
  Yuan, Jiaqi Wang, Conghui He, Ziwei Liu, et~al.
\newblock Mmbench: Is your multi-modal model an all-around player?
\newblock In \emph{European Conference on Computer Vision}, pages 216--233.
  Springer, 2025.

\bibitem[Long et~al.(2023)Long, Qin, Panteleev, Bissacco, Fujii, and
  Raptis]{hiertext}
Shangbang Long, Siyang Qin, Dmitry Panteleev, Alessandro Bissacco, Yasuhisa
  Fujii, and Michalis Raptis.
\newblock Icdar 2023 competition on hierarchical text detection and
  recognition.
\newblock \emph{arXiv preprint arXiv:2305.09750}, 2023.

\bibitem[Lu et~al.(2021)Lu, Qiu, Chen, Xia, Zhao, Zhang, Yu, Liang, and
  Zhu]{iconqa}
Pan Lu, Liang Qiu, Jiaqi Chen, Tony Xia, Yizhou Zhao, Wei Zhang, Zhou Yu,
  Xiaodan Liang, and Song-Chun Zhu.
\newblock Iconqa: A new benchmark for abstract diagram understanding and visual
  language reasoning.
\newblock \emph{arXiv preprint arXiv:2110.13214}, 2021.

\bibitem[Lu et~al.(2022)Lu, Mishra, Xia, Qiu, Chang, Zhu, Tafjord, Clark, and
  Kalyan]{scienceqa}
Pan Lu, Swaroop Mishra, Tony Xia, Liang Qiu, Kai-Wei Chang, Song-Chun Zhu,
  Oyvind Tafjord, Peter Clark, and Ashwin Kalyan.
\newblock Learn to explain: Multimodal reasoning via thought chains for science
  question answering.
\newblock In \emph{The 36th Conference on Neural Information Processing Systems
  (NeurIPS)}, 2022.

\bibitem[Lu et~al.(2023)Lu, Bansal, Xia, Liu, Li, Hajishirzi, Cheng, Chang,
  Galley, and Gao]{mathvista}
Pan Lu, Hritik Bansal, Tony Xia, Jiacheng Liu, Chunyuan Li, Hannaneh
  Hajishirzi, Hao Cheng, Kai-Wei Chang, Michel Galley, and Jianfeng Gao.
\newblock Mathvista: Evaluating mathematical reasoning of foundation models in
  visual contexts.
\newblock \emph{arXiv preprint arXiv:2310.02255}, 2023.

\bibitem[Lu et~al.(2024)Lu, Li, Chen, Xu, Luo, Zhang, and Ye]{ovis}
Shiyin Lu, Yang Li, Qing-Guo Chen, Zhao Xu, Weihua Luo, Kaifu Zhang, and
  Han-Jia Ye.
\newblock Ovis: Structural embedding alignment for multimodal large language
  model.
\newblock \emph{arXiv:2405.20797}, 2024.

\bibitem[Marino et~al.(2019)Marino, Rastegari, Farhadi, and Mottaghi]{okvqa}
Kenneth Marino, Mohammad Rastegari, Ali Farhadi, and Roozbeh Mottaghi.
\newblock Ok-vqa: A visual question answering benchmark requiring external
  knowledge.
\newblock In \emph{Proceedings of the IEEE/cvf conference on computer vision
  and pattern recognition}, pages 3195--3204, 2019.

\bibitem[Masry et~al.(2022)Masry, Long, Tan, Joty, and Hoque]{chartqa}
Ahmed Masry, Do~Xuan Long, Jia~Qing Tan, Shafiq Joty, and Enamul Hoque.
\newblock Chartqa: A benchmark for question answering about charts with visual
  and logical reasoning.
\newblock \emph{arXiv preprint arXiv:2203.10244}, 2022.

\bibitem[Mathew et~al.(2021{\natexlab{a}})Mathew, Gomez, Karatzas, and
  Jawahar]{squad}
Minesh Mathew, Lluis Gomez, Dimosthenis Karatzas, and CV Jawahar.
\newblock Asking questions on handwritten document collections.
\newblock \emph{International Journal on Document Analysis and Recognition
  (IJDAR)}, 24\penalty0 (3):\penalty0 235--249, 2021{\natexlab{a}}.

\bibitem[Mathew et~al.(2021{\natexlab{b}})Mathew, Karatzas, and
  Jawahar]{docvqa}
Minesh Mathew, Dimosthenis Karatzas, and CV Jawahar.
\newblock Docvqa: A dataset for vqa on document images.
\newblock In \emph{Proceedings of the IEEE/CVF winter conference on
  applications of computer vision}, pages 2200--2209, 2021{\natexlab{b}}.

\bibitem[Mathew et~al.(2022)Mathew, Bagal, Tito, Karatzas, Valveny, and
  Jawahar]{infographicvqa}
Minesh Mathew, Viraj Bagal, Rub{\`e}n Tito, Dimosthenis Karatzas, Ernest
  Valveny, and CV Jawahar.
\newblock Infographicvqa.
\newblock In \emph{Proceedings of the IEEE/CVF Winter Conference on
  Applications of Computer Vision}, pages 1697--1706, 2022.

\bibitem[Meng et~al.(2022)Meng, Sharma, Andonian, Belinkov, and Bau]{mass}
Kevin Meng, Arnab~Sen Sharma, Alex Andonian, Yonatan Belinkov, and David Bau.
\newblock Mass-editing memory in a transformer.
\newblock \emph{arXiv preprint arXiv:2210.07229}, 2022.

\bibitem[Methani et~al.(2020)Methani, Ganguly, Khapra, and Kumar]{plotqa}
Nitesh Methani, Pritha Ganguly, Mitesh~M Khapra, and Pratyush Kumar.
\newblock Plotqa: Reasoning over scientific plots.
\newblock In \emph{Proceedings of the IEEE/CVF Winter Conference on
  Applications of Computer Vision}, pages 1527--1536, 2020.

\bibitem[Mishra et~al.(2019)Mishra, Shekhar, Singh, and Chakraborty]{ocrvqa}
Anand Mishra, Shashank Shekhar, Ajeet~Kumar Singh, and Anirban Chakraborty.
\newblock Ocr-vqa: Visual question answering by reading text in images.
\newblock In \emph{ICDAR}, 2019.

\bibitem[Radford et~al.(2021)Radford, Kim, Hallacy, Ramesh, Goh, Agarwal,
  Sastry, Askell, Mishkin, Clark, et~al.]{clip}
Alec Radford, Jong~Wook Kim, Chris Hallacy, Aditya Ramesh, Gabriel Goh,
  Sandhini Agarwal, Girish Sastry, Amanda Askell, Pamela Mishkin, Jack Clark,
  et~al.
\newblock Learning transferable visual models from natural language
  supervision.
\newblock In \emph{International conference on machine learning}, pages
  8748--8763. PMLR, 2021.

\bibitem[Rajpurkar et~al.(2018)Rajpurkar, Jia, and Liang]{squad-v2}
Pranav Rajpurkar, Robin Jia, and Percy Liang.
\newblock Know what you don{'}t know: Unanswerable questions for {SQ}u{AD}.
\newblock In \emph{Proceedings of the 56th Annual Meeting of the Association
  for Computational Linguistics (Volume 2: Short Papers)}, pages 784--789,
  Melbourne, Australia, 2018. Association for Computational Linguistics.

\bibitem[Reid et~al.(2024)Reid, Savinov, Teplyashin, Lepikhin, Lillicrap,
  Alayrac, Soricut, Lazaridou, Firat, Schrittwieser, et~al.]{gemini}
Machel Reid, Nikolay Savinov, Denis Teplyashin, Dmitry Lepikhin, Timothy
  Lillicrap, Jean-baptiste Alayrac, Radu Soricut, Angeliki Lazaridou, Orhan
  Firat, Julian Schrittwieser, et~al.
\newblock Gemini 1.5: Unlocking multimodal understanding across millions of
  tokens of context.
\newblock \emph{arXiv preprint arXiv:2403.05530}, 2024.

\bibitem[Russakovsky et~al.(2015)Russakovsky, Deng, Su, Krause, Satheesh, Ma,
  Huang, Karpathy, Khosla, Bernstein, Berg, and Fei-Fei]{imagenet}
Olga Russakovsky, Jia Deng, Hao Su, Jonathan Krause, Sanjeev Satheesh, Sean Ma,
  Zhiheng Huang, Andrej Karpathy, Aditya Khosla, Michael Bernstein,
  Alexander~C. Berg, and Li Fei-Fei.
\newblock {ImageNet Large Scale Visual Recognition Challenge}.
\newblock \emph{International Journal of Computer Vision (IJCV)}, 115\penalty0
  (3):\penalty0 211--252, 2015.

\bibitem[Shah et~al.(2019)Shah, Mishra, Yadati, and Talukdar]{kvqa}
Sanket Shah, Anand Mishra, Naganand Yadati, and Partha~Pratim Talukdar.
\newblock Kvqa: Knowledge-aware visual question answering.
\newblock In \emph{Proceedings of the AAAI conference on artificial
  intelligence}, pages 8876--8884, 2019.

\bibitem[Shang et~al.(2024)Shang, Cai, Xu, Lee, and Yan]{prumerge}
Yuzhang Shang, Mu Cai, Bingxin Xu, Yong~Jae Lee, and Yan Yan.
\newblock Llava-prumerge: Adaptive token reduction for efficient large
  multimodal models.
\newblock \emph{arXiv preprint arXiv:2403.15388}, 2024.

\bibitem[Shao et~al.(2024)Shao, Qian, Xiao, Song, Zong, Wang, Liu, and
  Li]{visualcot}
Hao Shao, Shengju Qian, Han Xiao, Guanglu Song, Zhuofan Zong, Letian Wang, Yu
  Liu, and Hongsheng Li.
\newblock Visual cot: Unleashing chain-of-thought reasoning in multi-modal
  language models.
\newblock \emph{arXiv preprint arXiv:2403.16999}, 2024.

\bibitem[Shi et~al.(2024)Shi, Hu, Bin, Liu, Yang, Ng, Bing, and Lee]{mathv360k}
Wenhao Shi, Zhiqiang Hu, Yi Bin, Junhua Liu, Yang Yang, See-Kiong Ng, Lidong
  Bing, and Roy Ka-Wei Lee.
\newblock Math-llava: Bootstrapping mathematical reasoning for multimodal large
  language models.
\newblock \emph{arXiv preprint arXiv:2406.17294}, 2024.

\bibitem[Sidorov et~al.(2020)Sidorov, Hu, Rohrbach, and Singh]{textcaps}
Oleksii Sidorov, Ronghang Hu, Marcus Rohrbach, and Amanpreet Singh.
\newblock Textcaps: a dataset for image captioning with reading comprehension.
\newblock 2020.

\bibitem[Singh et~al.(2019)Singh, Natarajan, Shah, Jiang, Chen, Batra, Parikh,
  and Rohrbach]{textvqa}
Amanpreet Singh, Vivek Natarajan, Meet Shah, Yu Jiang, Xinlei Chen, Dhruv
  Batra, Devi Parikh, and Marcus Rohrbach.
\newblock Towards vqa models that can read.
\newblock In \emph{Proceedings of the IEEE/CVF conference on computer vision
  and pattern recognition}, pages 8317--8326, 2019.

\bibitem[Stammbach and Ash(2021)]{Docscan}
Dominik Stammbach and Elliott Ash.
\newblock Docscan: Unsupervised text classification via learning from
  neighbors.
\newblock \emph{arXiv preprint arXiv:2105.04024}, 2021.

\bibitem[Sujet~AI(2024)]{Sujet-Finance-QA-Vision-100k}
Hamed~Rahimi Sujet~AI, Allaa~Boutaleb.
\newblock Sujet-finance-qa-vision-100k: A large-scale dataset for financial
  document vqa, 2024.

\bibitem[Svetlichnaya(2020)]{deepform}
S Svetlichnaya.
\newblock Deepform: Understand structured documents at scale.
\newblock 2020.

\bibitem[Tanaka et~al.(2021)Tanaka, Nishida, and Yoshida]{visualmrc}
Ryota Tanaka, Kyosuke Nishida, and Sen Yoshida.
\newblock Visualmrc: Machine reading comprehension on document images.
\newblock In \emph{Proceedings of the AAAI Conference on Artificial
  Intelligence}, pages 13878--13888, 2021.

\bibitem[Tong et~al.(2024)Tong, Brown, Wu, Woo, Middepogu, Akula, Yang, Yang,
  Iyer, Pan, et~al.]{cambrian}
Shengbang Tong, Ellis Brown, Penghao Wu, Sanghyun Woo, Manoj Middepogu,
  Sai~Charitha Akula, Jihan Yang, Shusheng Yang, Adithya Iyer, Xichen Pan,
  et~al.
\newblock Cambrian-1: A fully open, vision-centric exploration of multimodal
  llms.
\newblock \emph{arXiv preprint arXiv:2406.16860}, 2024.

\bibitem[Touvron et~al.(2023)Touvron, Lavril, Izacard, Martinet, Lachaux,
  Lacroix, Rozi{\`e}re, Goyal, Hambro, Azhar, et~al.]{touvron2023llama}
Hugo Touvron, Thibaut Lavril, Gautier Izacard, Xavier Martinet, Marie-Anne
  Lachaux, Timoth{\'e}e Lacroix, Baptiste Rozi{\`e}re, Naman Goyal, Eric
  Hambro, Faisal Azhar, et~al.
\newblock Llama: Open and efficient foundation language models.
\newblock \emph{arXiv preprint arXiv:2302.13971}, 2023.

\bibitem[Veit et~al.(2016)Veit, Matera, Neumann, Matas, and
  Belongie]{coco-text}
Andreas Veit, Tomas Matera, Lukas Neumann, Jiri Matas, and Serge Belongie.
\newblock Coco-text: Dataset and benchmark for text detection and recognition
  in natural images.
\newblock \emph{arXiv preprint arXiv:1601.07140}, 2016.

\bibitem[Virmaux and Scaman(2018)]{lipschitz}
Aladin Virmaux and Kevin Scaman.
\newblock Lipschitz regularity of deep neural networks: analysis and efficient
  estimation.
\newblock \emph{Advances in Neural Information Processing Systems}, 31, 2018.

\bibitem[Wang et~al.(2024{\natexlab{a}})Wang, Bai, Tan, Wang, Fan, Bai, Chen,
  Liu, Wang, Ge, et~al.]{qwen2}
Peng Wang, Shuai Bai, Sinan Tan, Shijie Wang, Zhihao Fan, Jinze Bai, Keqin
  Chen, Xuejing Liu, Jialin Wang, Wenbin Ge, et~al.
\newblock Qwen2-vl: Enhancing vision-language model's perception of the world
  at any resolution.
\newblock \emph{arXiv preprint arXiv:2409.12191}, 2024{\natexlab{a}}.

\bibitem[Wang et~al.(2023)Wang, Lv, Yu, Hong, Qi, Wang, Ji, Yang, Zhao, Song,
  Xu, Xu, Li, Dong, Ding, and Tang]{cogvlm}
Weihan Wang, Qingsong Lv, Wenmeng Yu, Wenyi Hong, Ji Qi, Yan Wang, Junhui Ji,
  Zhuoyi Yang, Lei Zhao, Xixuan Song, Jiazheng Xu, Bin Xu, Juanzi Li, Yuxiao
  Dong, Ming Ding, and Jie Tang.
\newblock Cogvlm: Visual expert for pretrained language models, 2023.

\bibitem[Wang et~al.(2024{\natexlab{b}})Wang, He, Li, Li, Yu, Ma, Li, Chen,
  Chen, Wang, He, Luo, Liu, Wang, Wang, and Qiao]{sharegpt4o}
Yi Wang, Yinan He, Yizhuo Li, Kunchang Li, Jiashuo Yu, Xin Ma, Xinhao Li, Guo
  Chen, Xinyuan Chen, Yaohui Wang, Conghui He, Ping Luo, Ziwei Liu, Yali Wang,
  Limin Wang, and Yu Qiao.
\newblock Internvid: A large-scale video-text dataset for multimodal
  understanding and generation, 2024{\natexlab{b}}.

\bibitem[Wei et~al.(2023)Wei, Kong, Chen, Zhao, Ge, Yang, Sun, Han, and
  Zhang]{vary}
Haoran Wei, Lingyu Kong, Jinyue Chen, Liang Zhao, Zheng Ge, Jinrong Yang,
  Jian‐Yuan Sun, Chunrui Han, and Xiangyu Zhang.
\newblock Vary: Scaling up the vision vocabulary for large vision-language
  models.
\newblock \emph{ArXiv}, abs/2312.06109, 2023.

\bibitem[Wen et~al.(2024)Wen, Cao, Fu, Mehta, and Najibi]{victor}
Yuxin Wen, Qingqing Cao, Qichen Fu, Sachin Mehta, and Mahyar Najibi.
\newblock Efficient vision-language models by summarizing visual tokens into
  compact registers.
\newblock \emph{arXiv preprint arXiv:2410.14072}, 2024.

\bibitem[Wendler(2023)]{chromewriting}
C. Wendler.
\newblock Renderedtext.
\newblock \url{https://huggingface.co/datasets/wendlerc/RenderedText}, 2023.

\bibitem[Wenhu~Chen and Wang(2020)]{tabfact}
Jianshu Chen Yunkai Zhang Hong Wang Shiyang Li Xiyou~Zhou Wenhu~Chen,
  Hongmin~Wang and William~Yang Wang.
\newblock Tabfact : A large-scale dataset for table-based fact verification.
\newblock In \emph{International Conference on Learning Representations
  (ICLR)}, Addis Ababa, Ethiopia, 2020.

\bibitem[Xing et~al.(2024)Xing, Huang, Dong, Lu, Zhang, Zang, Cao, He, Wang,
  Wu, et~al.]{pyramiddrop}
Long Xing, Qidong Huang, Xiaoyi Dong, Jiajie Lu, Pan Zhang, Yuhang Zang, Yuhang
  Cao, Conghui He, Jiaqi Wang, Feng Wu, et~al.
\newblock Pyramiddrop: Accelerating your large vision-language models via
  pyramid visual redundancy reduction.
\newblock \emph{arXiv preprint arXiv:2410.17247}, 2024.

\bibitem[Xu et~al.(2024)Xu, Yao, Guo, Cui, Ni, Ge, Chua, Liu, Sun, and
  Huang]{llava-uhd}
Ruyi Xu, Yuan Yao, Zonghao Guo, Junbo Cui, Zanlin Ni, Chunjiang Ge, Tat-Seng
  Chua, Zhiyuan Liu, Maosong Sun, and Gao Huang.
\newblock Llava-uhd: an lmm perceiving any aspect ratio and high-resolution
  images.
\newblock \emph{arXiv preprint arXiv:2403.11703}, 2024.

\bibitem[Xu et~al.(2023)Xu, Du, Qi, Xu, Yuan, and Guo]{chartbench}
Zhengzhuo Xu, Sinan Du, Yiyan Qi, Chengjin Xu, Chun Yuan, and Jian Guo.
\newblock Chartbench: A benchmark for complex visual reasoning in charts.
\newblock \emph{arXiv preprint arXiv:2312.15915}, 2023.

\bibitem[Yang et~al.(2024{\natexlab{a}})Yang, Yang, Hui, Zheng, Yu, Zhou, Li,
  Li, Liu, Huang, et~al.]{qwen2llm}
An Yang, Baosong Yang, Binyuan Hui, Bo Zheng, Bowen Yu, Chang Zhou, Chengpeng
  Li, Chengyuan Li, Dayiheng Liu, Fei Huang, et~al.
\newblock Qwen2 technical report.
\newblock \emph{arXiv preprint arXiv:2407.10671}, 2024{\natexlab{a}}.

\bibitem[Yang et~al.(2024{\natexlab{b}})Yang, Chen, Tian, Wang, Li, Yu, and
  Jia]{visionzip}
Senqiao Yang, Yukang Chen, Zhuotao Tian, Chengyao Wang, Jingyao Li, Bei Yu, and
  Jiaya Jia.
\newblock Visionzip: Longer is better but not necessary in vision language
  models.
\newblock \emph{arXiv preprint arXiv:2412.04467}, 2024{\natexlab{b}}.

\bibitem[Yao et~al.(2024{\natexlab{a}})Yao, Li, Ren, Wang, Liu, Sun, and
  Hou]{deco}
Linli Yao, Lei Li, Shuhuai Ren, Lean Wang, Yuanxin Liu, Xu Sun, and Lu Hou.
\newblock Deco: Decoupling token compression from semantic abstraction in
  multimodal large language models.
\newblock \emph{arXiv preprint arXiv:2405.20985}, 2024{\natexlab{a}}.

\bibitem[Yao et~al.(2024{\natexlab{b}})Yao, Yu, Zhang, Wang, Cui, Zhu, Cai, Li,
  Zhao, He, et~al.]{minicpm}
Yuan Yao, Tianyu Yu, Ao Zhang, Chongyi Wang, Junbo Cui, Hongji Zhu, Tianchi
  Cai, Haoyu Li, Weilin Zhao, Zhihui He, et~al.
\newblock Minicpm-v: A gpt-4v level mllm on your phone.
\newblock \emph{arXiv preprint arXiv:2408.01800}, 2024{\natexlab{b}}.

\bibitem[Ye et~al.(2024{\natexlab{a}})Ye, Xu, Liu, Hu, Yan, Qian, Zhang, Huang,
  and Zhou]{mplug-owl3}
Jiabo Ye, Haiyang Xu, Haowei Liu, Anwen Hu, Ming Yan, Qi Qian, Ji Zhang, Fei
  Huang, and Jingren Zhou.
\newblock mplug-owl3: Towards long image-sequence understanding in multi-modal
  large language models.
\newblock \emph{arXiv preprint arXiv:2408.04840}, 2024{\natexlab{a}}.

\bibitem[Ye et~al.(2024{\natexlab{b}})Ye, Gan, Huang, Ge, Shan, and Tang]{voco}
Xubing Ye, Yukang Gan, Xiaoke Huang, Yixiao Ge, Ying Shan, and Yansong Tang.
\newblock Voco-llama: Towards vision compression with large language models.
\newblock \emph{arXiv preprint arXiv:2406.12275}, 2024{\natexlab{b}}.

\bibitem[Yu et~al.(2024)Yu, Zhang, Yao, Dang, Chen, Lu, Cui, He, Liu, Chua, and
  Sun]{rlaifv}
Tianyu Yu, Haoye Zhang, Yuan Yao, Yunkai Dang, Da Chen, Xiaoman Lu, Ganqu Cui,
  Taiwen He, Zhiyuan Liu, Tat-Seng Chua, and Maosong Sun.
\newblock Rlaif-v: Aligning mllms through open-source ai feedback for super
  gpt-4v trustworthiness.
\newblock \emph{arXiv preprint arXiv:2405.17220}, 2024.

\bibitem[Yu et~al.(2023)Yu, Yang, Li, Wang, Lin, Liu, Wang, and Wang]{mmvet}
Weihao Yu, Zhengyuan Yang, Linjie Li, Jianfeng Wang, Kevin Lin, Zicheng Liu,
  Xinchao Wang, and Lijuan Wang.
\newblock Mm-vet: Evaluating large multimodal models for integrated
  capabilities.
\newblock \emph{arXiv preprint arXiv:2308.02490}, 2023.

\bibitem[Yu et~al.(2022)Yu, Min, Zettlemoyer, and Hajishirzi]{crepe}
Xinyan~Velocity Yu, Sewon Min, Luke Zettlemoyer, and Hannaneh Hajishirzi.
\newblock Crepe: Open-domain question answering with false presuppositions.
\newblock \emph{arXiv preprint arXiv:2211.17257}, 2022.

\bibitem[Yuan et~al.(2025)Yuan, Zhang, Liu, Chen, Lu, Lin, Zheng, Han, and
  Sun]{Shortv}
Qianhao Yuan, Qingyu Zhang, Yanjiang Liu, Jiawei Chen, Yaojie Lu, Hongyu Lin,
  Jia Zheng, Xianpei Han, and Le Sun.
\newblock Shortv: Efficient multimodal large language models by freezing visual
  tokens in ineffective layers.
\newblock \emph{arXiv preprint arXiv:2504.00502}, 2025.

\bibitem[Yuan et~al.(2023)Yuan, Li, Huang, Ye, and Sun]{tinygpt}
Zhengqing Yuan, Zhaoxu Li, Weiran Huang, Yanfang Ye, and Lichao Sun.
\newblock Tinygpt-v: Efficient multimodal large language model via small
  backbones.
\newblock \emph{arXiv preprint arXiv:2312.16862}, 2023.

\bibitem[Yue et~al.(2024)Yue, Ni, Zhang, Zheng, Liu, Zhang, Stevens, Jiang,
  Ren, Sun, et~al.]{mmmu}
Xiang Yue, Yuansheng Ni, Kai Zhang, Tianyu Zheng, Ruoqi Liu, Ge Zhang, Samuel
  Stevens, Dongfu Jiang, Weiming Ren, Yuxuan Sun, et~al.
\newblock Mmmu: A massive multi-discipline multimodal understanding and
  reasoning benchmark for expert agi.
\newblock In \emph{Proceedings of the IEEE/CVF Conference on Computer Vision
  and Pattern Recognition}, pages 9556--9567, 2024.

\bibitem[Zhang et~al.(2019)Zhang, Gao, Jia, Zhu, and Zhu]{raven}
Chi Zhang, Feng Gao, Baoxiong Jia, Yixin Zhu, and Song-Chun Zhu.
\newblock Raven: A dataset for relational and analogical visual reasoning.
\newblock In \emph{Proceedings of the IEEE/CVF conference on computer vision
  and pattern recognition}, pages 5317--5327, 2019.

\bibitem[Zhang et~al.(2024{\natexlab{a}})Zhang, Cheng, Lu, Zhuo, Wang, Cao,
  Guo, She, and Zhang]{fastvlm}
Qizhe Zhang, Aosong Cheng, Ming Lu, Zhiyong Zhuo, Minqi Wang, Jiajun Cao,
  Shaobo Guo, Qi She, and Shanghang Zhang.
\newblock [cls] attention is all you need for training-free visual token
  pruning: Make vlm inference faster.
\newblock \emph{arXiv preprint arXiv:2412.01818}, 2024{\natexlab{a}}.

\bibitem[Zhang et~al.(2024{\natexlab{b}})Zhang, Jiang, Zhang, Lin, Guo, Qiu,
  Zhou, Lu, Chang, Gao, et~al.]{mathverse}
Renrui Zhang, Dongzhi Jiang, Yichi Zhang, Haokun Lin, Ziyu Guo, Pengshuo Qiu,
  Aojun Zhou, Pan Lu, Kai-Wei Chang, Peng Gao, et~al.
\newblock Mathverse: Does your multi-modal llm truly see the diagrams in visual
  math problems?
\newblock \emph{arXiv preprint arXiv:2403.14624}, 2024{\natexlab{b}}.

\bibitem[Zhang et~al.(2017)Zhang, Gueguen, Zharkov, Zhang, Seifert, and
  Kadlec]{Uber-text}
Ying Zhang, Lionel Gueguen, Ilya Zharkov, Peter Zhang, Keith Seifert, and Ben
  Kadlec.
\newblock Uber-text: A large-scale dataset for optical character recognition
  from street-level imagery.
\newblock In \emph{SUNw: Scene Understanding Workshop-CVPR}, page~5, 2017.

\bibitem[Zhang et~al.(2024{\natexlab{c}})Zhang, Fan, Ma, Zheng, Huang, Cheng,
  Gudovskiy, Okuno, Nakata, Keutzer, et~al.]{sparsevlm}
Yuan Zhang, Chun-Kai Fan, Junpeng Ma, Wenzhao Zheng, Tao Huang, Kuan Cheng,
  Denis Gudovskiy, Tomoyuki Okuno, Yohei Nakata, Kurt Keutzer, et~al.
\newblock Sparsevlm: Visual token sparsification for efficient vision-language
  model inference.
\newblock \emph{arXiv preprint arXiv:2410.04417}, 2024{\natexlab{c}}.

\bibitem[Zhang et~al.(2023)Zhang, Sheng, Zhou, Chen, Zheng, Cai, Song, Tian,
  R{\'e}, Barrett, et~al.]{h2o}
Zhenyu Zhang, Ying Sheng, Tianyi Zhou, Tianlong Chen, Lianmin Zheng, Ruisi Cai,
  Zhao Song, Yuandong Tian, Christopher R{\'e}, Clark Barrett, et~al.
\newblock H2o: Heavy-hitter oracle for efficient generative inference of large
  language models.
\newblock \emph{Advances in Neural Information Processing Systems},
  36:\penalty0 34661--34710, 2023.

\bibitem[Zhao et~al.(2023)Zhao, Wu, He, and Huang]{svit}
Bo Zhao, Boya Wu, Muyang He, and Tiejun Huang.
\newblock Svit: Scaling up visual instruction tuning.
\newblock \emph{arXiv preprint arXiv:2307.04087}, 2023.

\bibitem[Zhou et~al.(2024)Zhou, Hu, Weng, Jia, Luo, Liu, Wu, and
  Huang]{tinyllava}
Baichuan Zhou, Ying Hu, Xi Weng, Junlong Jia, Jie Luo, Xien Liu, Ji Wu, and Lei
  Huang.
\newblock Tinyllava: A framework of small-scale large multimodal models.
\newblock \emph{arXiv preprint arXiv:2402.14289}, 2024.

\bibitem[Zhu et~al.(2023)Zhu, Chen, Shen, Li, and Elhoseiny]{minigpt}
Deyao Zhu, Jun Chen, Xiaoqian Shen, Xiang Li, and Mohamed Elhoseiny.
\newblock Minigpt-4: Enhancing vision-language understanding with advanced
  large language models.
\newblock \emph{arXiv preprint arXiv:2304.10592}, 2023.

\bibitem[Zhu et~al.(2021)Zhu, Lei, Huang, Wang, Zhang, Lv, Feng, and Chua]{tat}
Fengbin Zhu, Wenqiang Lei, Youcheng Huang, Chao Wang, Shuo Zhang, Jiancheng Lv,
  Fuli Feng, and Tat-Seng Chua.
\newblock Tat-qa: A question answering benchmark on a hybrid of tabular and
  textual content in finance.
\newblock \emph{arXiv preprint arXiv:2105.07624}, 2021.

\end{thebibliography}
}
\clearpage
\setcounter{page}{1}
\maketitlesupplementary

\section{Proof and detailed analysis of Skip-Vision}
\label{sec:analysis}
In this section, we provide a theoretical analysis to justify the rationale behind Skip-Vision and quantify the performance loss introduced by skipping FFN computations for redundant visual tokens.

\subsection{Bounded error of Skip FFN}
At the core of this analysis lies the layer error incurred when bypassing the FFN layer. For a given layer $l$,  the original output $h_{\text{original}}^{(l)}$ is :
\begin{equation}
h_{\text{original}}^{(l)}=h_{\text{attn}}^{(l)}+\text{FFN}^{(l)}(h_{\text{attn}}^{(l)}).
\end{equation}
When skipping the FFN, the output becomes:
\begin{equation}
h_{\text{skip}}^{(l)}=h_{\text{attn}}^{(l)}.
\end{equation}
The per-layer skipping error is:
\begin{equation}
\epsilon^{(l)}=\|h_{\text{original}}^{(l)}-h_{\text{skip}}^{(l)}\|_2=\|\text{FFN}^{(l)}(h_{\text{attn}}^{(l)})\|_2.
\end{equation}
For redundant tokens, such as those in homogeneous image regions, this error is negligible $(\epsilon^{(l)}\approx0)$ due to minimal feature transformations by the FFN.

However, errors propagate through subsequent layers, amplified by the recursive nature of transformer architectures. Leveraging Lipschitz continuity assumptions for self-attention and FFN operations ($L^{(l+1)}_{attn}$ and $L^{(l+1)}_{FFN}$), the cumulative error at layer $l+1$ is bounded by:
\begin{equation}
\epsilon^{(l + 1)}\leq (L^{(l+1)}_{\text{attn}}+L^{(l+1)}_{\text{FFN}})\cdot\epsilon^{(l)}+\epsilon_{\text{skip}}^{(l + 1)},
\end{equation}
where $(\epsilon_{\text{skip}}^{(l + 1)}$ represents new errors from skipping deeper layers $(l + 1)$.

Over $L$ layers, the total error telescopes to:
\begin{equation}
\epsilon_{\text{total}} \leq \sum_{l = 1}^{L}\epsilon_{\text{skip}}^{(l)} \cdot \prod_{i = 1}^{L-l}(L^{(i+1)}_{\text{attn}}+L^{(i+1)}_{\text{FFN}}).
\end{equation}

\begin{theorem}
\textbf{Lipschitz Constants for Causal Attention and FFN in Transformers}
Assume:
\begin{enumerate}
    \item Inputs are normalized (e.g., via LayerNorm), bounding intermediate feature norms.
    \item Weight matrices in attention ($W_Q$, $W_K$, $W_V$) and FFN ($W_1$, $W_2$) have bounded spectral norms (maximum singular values).
\end{enumerate}
Then the Lipschitz constants of these components satisfy:
\begin{enumerate}
    \item Causal Attention:
    \begin{equation}
    \text{L}(\text{Attn}) \leq \frac{\|W_Q\|_2\|W_K\|_2\|W_V\|_2}{\sqrt{d_k}},
    \end{equation}
where $d_k$ is the key dimension. The causal mask further restricts attention dependencies, preserving this bound.
    \item Feed - Forward Network (FFN):
    \begin{equation}
\text{L}(\text{FFN}) \leq \|W_1\|_2\|W_2\|_2,
    \end{equation}
assuming the activation function (e.g., ReLU, GELU) is 1-Lipschitz.
\end{enumerate}
\end{theorem}
\begin{proof}

\textbf{The Lipschitz constant of causal attention} 

\textbf{1.Linear Transformation:}

The input sequence $X \in \mathbb{R}^{n \times d}$ undergoes three linear transformations to obtain the query $Q = XW_Q$, the key $K = XW_K$, and the value $V = XW_V$. The Lipschitz constant of each linear transformation is the spectral norm (the largest singular value) of its weight matrix, denoted as $\sigma_Q=\|W_Q\|_2$, $\sigma_K = \|W_K\|_2$, and $\sigma_V=\|W_V\|_2$ respectively.

\textbf{2.Attention Score Calculation:}

The scaled dot - product $S = QK^{\top}/\sqrt{d_k}$. The Lipschitz constant of the bilinear mapping is related to $\sigma_Q$ and $\sigma_K$. If the norm of the input $X$ is bounded (for example, through LayerNorm), then:
\begin{equation}
\text{Lip}(S) \leq \frac{\sigma_Q\sigma_K}{\sqrt{d_k}}.
\end{equation}

\textbf{3. Softmax Activation:}

After applying the causal mask, softmax is performed on each row. The Lipschitz constant of Softmax under the $\ell_2$ norm is less than 1, that is:
\begin{equation}
\text{Lip}(\text{softmax}) \leq 1.
\end{equation}

\textbf{4. Weighted Sum of Values:} The output $\text{Attn}(X)=AV$, where $A = \text{softmax}(S)$. The Lipschitz constant of this step is determined by the spectral norm $\sigma_V$ of the linear transformation of $V$.

\textbf{Overall Lipschitz Constant}

Combining the upper bounds of each step:
\begin{equation}
\text{Lip}(\text{CausalAttention}) \leq \frac{\sigma_Q\sigma_K\sigma_V}{\sqrt{d_k}}.
\end{equation}

\textbf{Impact of Causal Mask:} The mask restricts the attention range, which may reduce the sensitivity to the input. Therefore, the actual Lipschitz constant will not exceed the above-mentioned upper bound.

\textbf{The Lipschitz Constant of FFN}

The FFN is usually expressed as:
\begin{equation}
\text{FFN}(x)=W_2\cdot\text{Activation}(W_1x + b_1)+b_2,
\end{equation}
where the Lipschitz constant of activation functions (such as ReLU, GELU) is 1.

Derivation of Lipschitz Constant

1. Linear Layer \(W_1\): The spectral norm is \(\sigma_1=\|W_1\|_2\).

2. Activation Function: \(\text{Lip}(\text{Activation}) = 1\).

3. Linear Layer \(W_2\): The spectral norm is \(\sigma_2=\|W_2\|_2\).

\textbf{Follow the equation 7 in \cite{lipschitz}}, the overall Lipschitz constant is the product of the spectral norms of the two linear layers:
\begin{equation}
\text{Lip}(\text{FFN})\leq\sigma_1\sigma_2.
\end{equation}
\end{proof}
\begin{corollary}
\textbf{Bounded Lipschitz}
If $W_Q$, $W_K$, $W_V$, $W_1$, $W_2$ are orthogonal matrices (spectral norm = 1), then:
\begin{itemize}
    \item $\text{L}(\text{attn})\leq1/\sqrt{d_k}$
    \item $\text{L}(\text{FFN})\leq1$
\end{itemize}
\end{corollary}
If we assume Lipschitz constants $L_{\text{attn}}+L_{\text{FFN}} \leq \gamma$ and skipping errors $\epsilon_{\text{skip}}^{(l)} \leq \epsilon$, the total error is scaled to:
\begin{equation}
\epsilon_{\text{total}} \leq \epsilon\cdot\frac{\gamma^{L}-1}{\gamma - 1},
\end{equation}
Theorem \ref{cor:1} establish that the skip error is bounded when
$\gamma<1$, provided the model is trained with modern regularization techniques. This ensures that the Multimodal Large Language Model (MLLM) remains less sensitive to the effects of skipping.

This error also impacts the KL divergence between the original and skipped outputs, bounded by:
\begin{equation}
\mathcal{D}_{\text{KL}}(p_{\text{skip}}\parallel p_{\text{original}})\leq\frac{1}{2\sigma^{2}}\cdot\epsilon_{\text{total}}^{2},
\end{equation}
where $\sigma^{2}$ is the variance of the logits. 

\begin{proof}
For two Gaussian distributions $p = \mathcal{N}(\mu_p, \Sigma_p)$ and $q = \mathcal{N}(\mu_q, \Sigma_q)$, their KL divergence is:
\begin{align}\nonumber
&\mathcal{D}_{\text{KL}}(p \| q)\\\nonumber
&=\frac{1}{2}(\text{tr}(\Sigma_q^{-1}\Sigma_p)+(\mu_q - \mu_p)^T\Sigma_q^{-1}(\mu_q - \mu_p)\\
&-k+\ln\frac{|\Sigma_q|}{|\Sigma_p|})
\end{align}
where $k$ is the dimension. If we assume that the covariances of the two distributions are the same, i.e., $\Sigma_p=\Sigma_q = \sigma^2 I$, and the total difference in means is $\epsilon_{\text{total}}$, then:
\begin{equation}
\mathcal{D}_{\text{KL}}(p \| q)=\frac{1}{2\sigma^2}\|\mu_p - \mu_q\|^2.
\end{equation}
Here, $\|\mu_p - \mu_q\|^2$ is $\epsilon_{\text{total}}^2$, so:
\begin{equation}
\mathcal{D}_{\text{KL}}(p \| q)\leq\frac{1}{2\sigma^2}\epsilon_{\text{total}}^2.
\end{equation}
\end{proof}

Further integrating feature similarity errors $(\epsilon_{\text{sim}} = O(\sqrt{1 - \theta}))$ from low-attention tokens, the final bound becomes:
\begin{equation}
\mathcal{D}_{\text{KL}}\leq\frac{1}{2\sigma^{2}}\cdot\left(\epsilon_{\text{total}}+\epsilon_{\text{sim}}\right)^{2}.
\end{equation}
Practically, this analysis motivates a layer-wise skipping strategy, alongside token selection and token merge based on feature similarity $(\theta)$.

\section{More experimental results}
\subsection{Efficiency.} 
\label{llava-1.5-exp}
Following the LLaVA-1.5-7B training setup, we conducted additional comparisons between Skip-Vision and several recent works, as shown in the table \ref{tab:long_vision}.
MMVet, MMStar and MMBench highlight Skip-Vision’s strength in \textbf{capturing causal and global information}. These benchmarks emphasize high-level reasoning and abstraction, which benefit from Skip-Vision’s ability to \textbf{preserve essential information flow} while reducing redundant computations. By skipping FFN and KV-cache for less informative tokens, the model \textbf{amplifies signal from key visual cues} and enhances \textbf{causal token interactions}.
While this comes with a \textbf{slight trade-off} in fine-grained tasks (OCR, Textvqa), it reflects a deliberate balance between perception and reasoning, favoring tasks that rely on semantic integration over detail fidelity.
\begin{table}[h]
    \centering
    \vspace{-2mm}
    \resizebox{\linewidth}{!}{
    \begin{tabular}{l|ccccc}
    \toprule
    Method & GQA & MMB & VQA$^{\text{Text}}$ & MMVet & Avg. \\
    Vanilla (576 tokens) & 61.9 & 64.7 & 58.2 & 31.1 & 100\% \\
    \midrule
    SparseVLM~\cite{sparsevlm} (64 tokens) & 52.7 & 56.2 & 51.8 & 23.3 & 85.2\% \\
    VisionZip~\cite{visionzip} (64 tokens) & 55.1 & 60.1 & 55.5 & 31.7 & 93.7\% \\
    PDrop~\cite{pyramiddrop} (64 tokens) & 47.5 & 58.8 & 50.6 & - & - \\
     FasterVLM~\cite{fastvlm} (58 tokens) & 54.9 & 60.6 & 55.3 & 30.1 & 93.1\% \\
     LLaVA-PruMerge~\cite{prumerge} (32 tokens) & - & 60.9 & 56.0 & - & - \\
     Skip-Vision ($N_r=64, N_s=156$) & \textbf{60.8} & \textbf{65.1} & \textbf{57.4} & \textbf{32.5} & \textbf{100.0}\% \\
    \bottomrule
    \end{tabular}
    }
    \caption{Comparison with more methods}
    \label{tab:long_vision}
    \vspace{-5mm}
\end{table}

Under the cos setting, we conduct more experiments to evaluate the training and inference efficiency. We compare with methods: FastV~\cite{fastv}, Victor~\cite{victor} and mean average pool, fine-tuning on LLaVA-665k using 8 NVIDIA A100 GPUs. As shown in Figure \ref{fig:5}, our architecture outperforms in both metrics. Compared to the $CoS_{1296}$ baseline, it achieves comparable performance with $35\%$ less training time and $74\%$ reduced inference computation. FastV, unable to utilize flash attention, shows a significant disadvantage, even surpassing baseline training time.
\begin{figure}[htbp]
  \centering
  \setlength{\abovecaptionskip}{0.cm}
   \includegraphics[width=\linewidth]{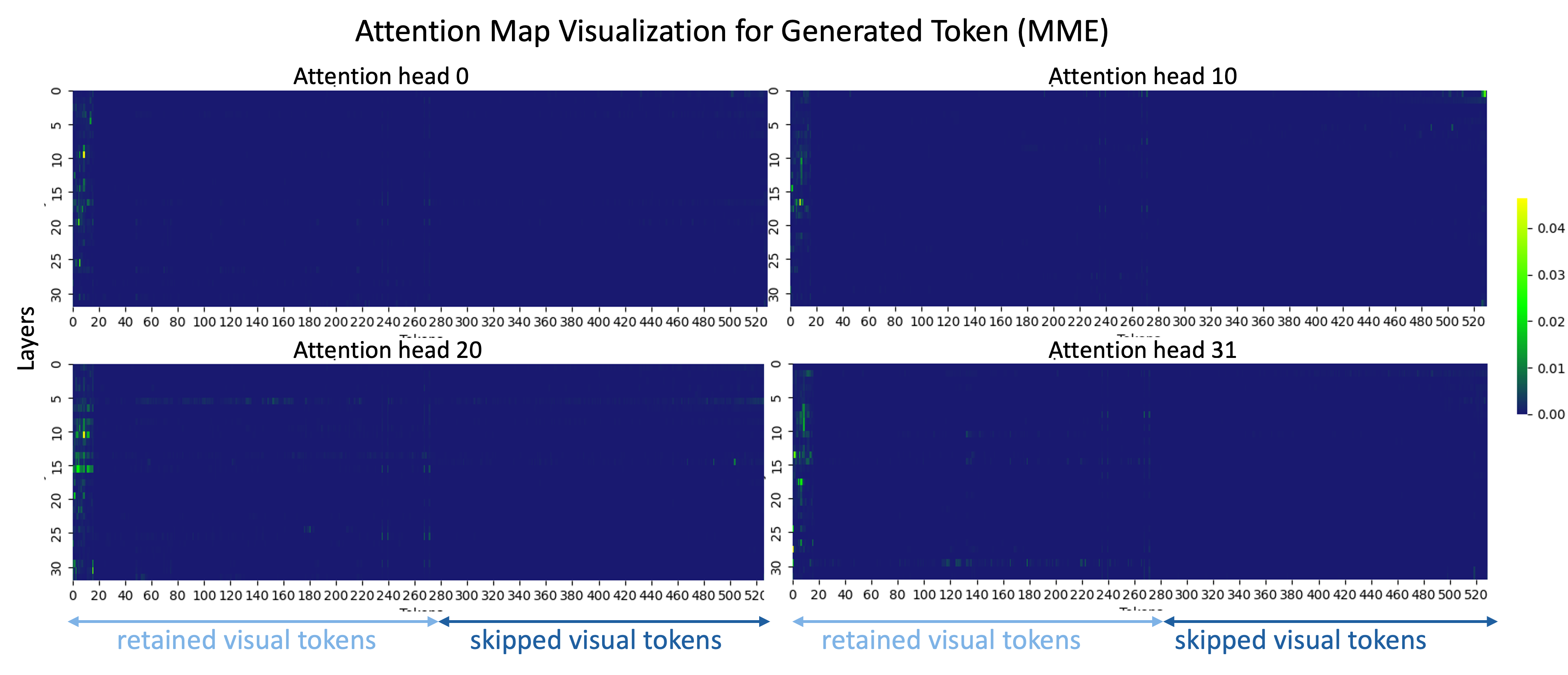}
   \caption{\textbf{Visualization of attention map in MME.}}
   \label{fig:mme}
   \vspace{-3mm}
\end{figure}

\begin{figure}[htbp]
  \centering
  \setlength{\abovecaptionskip}{0.cm}
   \includegraphics[width=\linewidth]{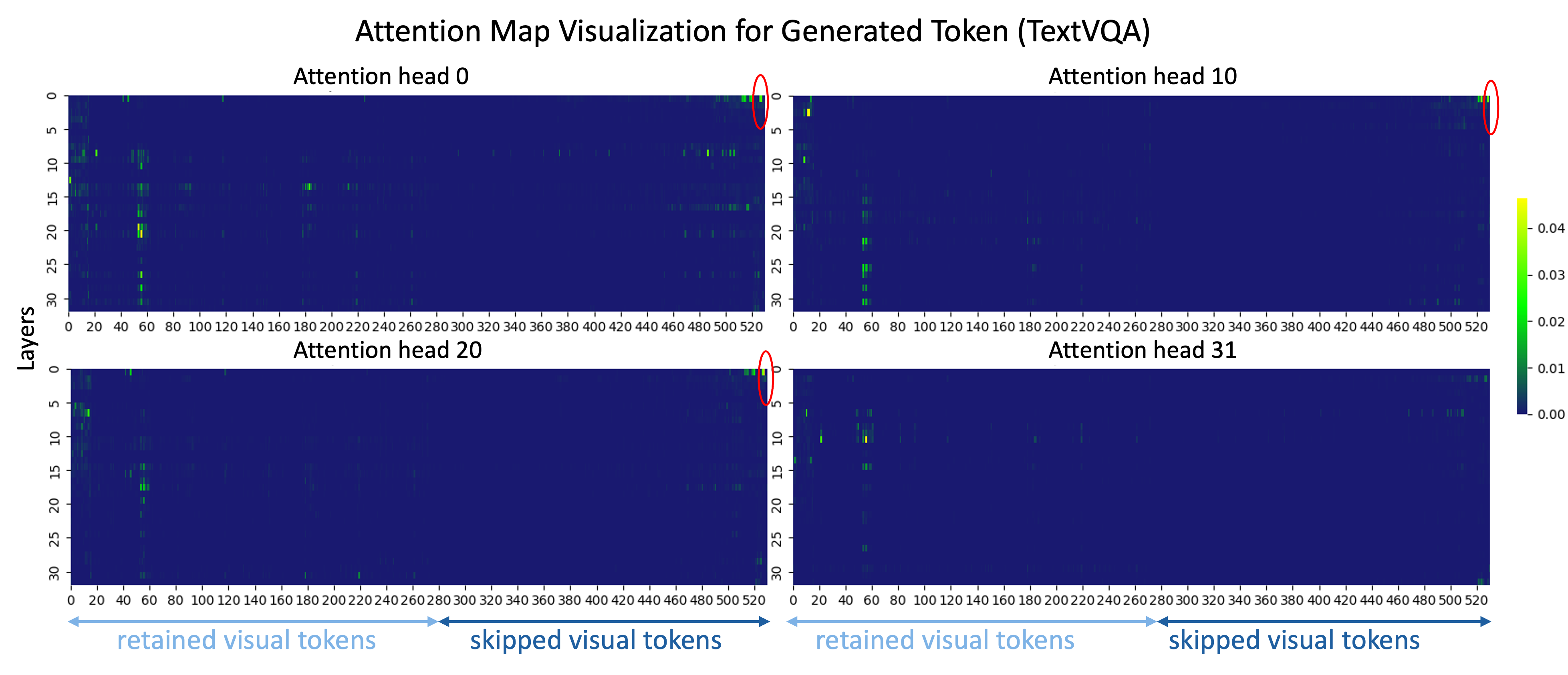}
   \caption{\textbf{Visualization of attention map in TextVQA.}}
   \label{fig:textvqa}
   \vspace{-5mm}
\end{figure}
\begin{figure}[htp]
  \centering
   \includegraphics[width=0.9\linewidth]{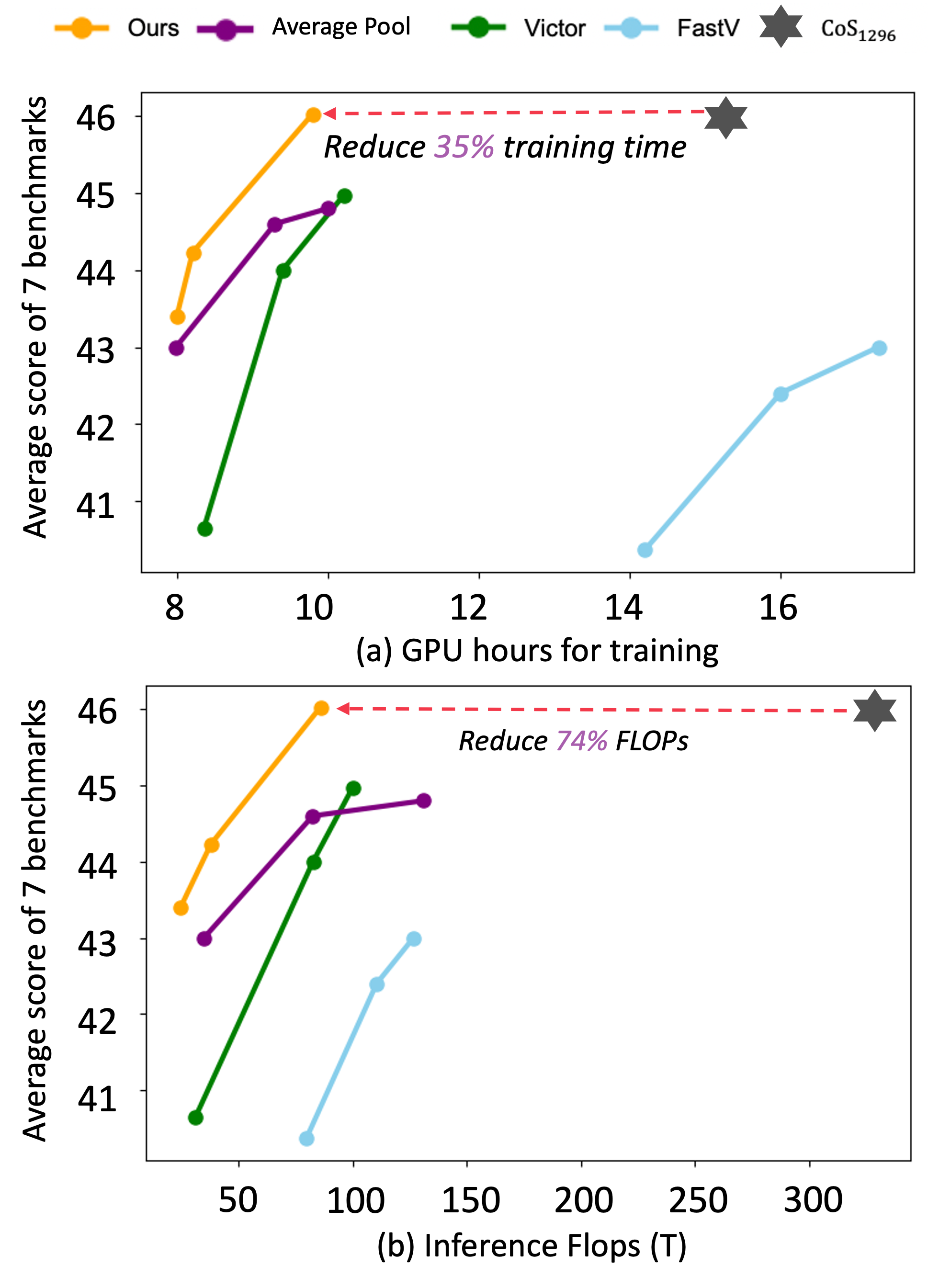}
   \vspace{-3mm}
   \caption{\textbf{Performance vs. Training-Time and Inference FLOPs.} Under the cos setting, we compare Skip-Vision with three MLLM acceleration methods, showing clear advantages in training speed and inference efficiency under equivalent computational constraints.}
   \label{fig:5}
   \vspace{-5mm}
\end{figure}
\subsection{Ablation study} 
\begin{table*}
  \centering
  \setlength\tabcolsep{3pt}
\begin{tabular}{ccccccc|ccccccccc} 
& SF & FS & LS & Merge & LV & SK & MME & Textvqa & MMB & MMVet & MMMU & MathV & OCRB & MMStar & Overall \\
\hline
0& $CoS_{1296}$& & & & & & 1585 & 64.4 & 77.1 & 39.4 & 39.2 & 21.5 & 39.2 & 41.2 & 46 \\
\hline
1& & & &$\checkmark$ & & & 1548 &  64.5 & 75.3  & 39.7  &39.1 &21.8  &37.7 &40.9 &45.6 \\
2& $\checkmark$& & & & & & 1589 & 63.6  & 74.4  & 36.9  & 38 & 19.8 & 365 &39.8& 44.2 \\
3& $\checkmark$& & & &$\checkmark$ & & 1591 &63.7 & 74.9&39.4 &38.7 &20.7 &36.3 & 39.9&44.8 \\
4& $\checkmark$& &$\checkmark$ & & & & 1560 &63.8 & 75.5&39.0 &39.2 &21 &37.1 & 40.0&45.0 \\
5& $\checkmark$& & $\checkmark$& & $\checkmark$ & & 1580 & 63.5 & 74.2 & 40.6 & 39.2 & 21.6 & 37.2 & 40.9 & 45.3 \\
6& $\checkmark$& $\checkmark$ & $\checkmark$& & & & 1570 & 63.5 & 74.1 & 40 & 39.8 & 20.1 & 39 & 41.4 & 45.4 \\
7& $\checkmark$ & &$\checkmark$ &$\checkmark$ & & & 1593 &62.6 & 74.7 & 40.3 & 40.2 & 21.6 & 36.6 & 41.1 & 45.3 \\
8& $\checkmark$ & $\checkmark$ &$\checkmark$ &$\checkmark$ & & & 1571&63.6 & 75.9 & 40.3 & 40.3 & 21 & 36.1 & 41.5 & 45.6 \\
9& $\checkmark$ & &$\checkmark$ & $\checkmark$ & $\checkmark$ & & 1547 & 64.0 & 75.9 & 38 & 41.2 & 21.9  & 36.9  & 40.6  &  45.5 \\
10& $\checkmark$ & $\checkmark$ & $\checkmark$ & $\checkmark$& $\checkmark$& & 1562 & 63.9 & 76.5 & 40.2 & 40.2 & 21.2 & 37.2 & 41.9 & 45.9 \\
11& $\checkmark$ & $\checkmark$ & $\checkmark$ & $\checkmark$& $\checkmark$& $\checkmark$ & 1563 & 63.7 & 76.5 & 41.7 & 40.3 & 21.2 & 37.0 & 41.9 & 46 \\
\hline
\end{tabular}
  \caption{
  \textbf{Ablation study.} To establish a strong baseline, we performed an ablation study on each component of the Skip-Vision framework with the LLAVA 665k SFT dataset. SF (skip FFN), FS (former summary token), LS (latter summary token), Merge (reducing local visual tokens from 1024 to 256), LV (passing the last local visual token through the FFN), SK (using skip KV-cache during inference). This analysis highlights the distinct contributions of each element to efficiency and performance.}
  \label{tab:3}
  \vspace{-3mm}
\end{table*}

\begin{table*}
  \centering
  \setlength\tabcolsep{3pt}
\begin{tabular}{cc|ccccccccc} 
Inference method & Skip window size & MME & Textvqa & MMB & MMVet & MMMU & MathV & OCRB & MMStar & Overall \\
\hline
Without skip KV-cache & - & 1562 & \textbf{63.9} & \textbf{76.5} & 40.2 & 40.2 & \textbf{21.2} & \textbf{37.2} & \textbf{41.9} & 45.9 \\
Skip KV-cache & middle+small & 1562 & 61.8 & 76.5 & 32.6 & \textbf{40.4} & \textbf{21.2} & 22.6 & \textbf{41.9} & 42.4 \\
Skip KV-cache & small & \textbf{1563} & 63.7 & \textbf{76.5} & \textbf{41.7} & 40.3 & \textbf{21.2} & 37.0 & \textbf{41.9} & \textbf{46} \\
\hline
\end{tabular}
  \caption{ 
  \textbf{Ablation study of skip KV-cache.}We report skip KV-cache performance across different visual token window sizes. Skip-Vision enables task-specific optimization by adjusting skip KV-cache levels for tailored acceleration.}
  \vspace{-3mm}
  \label{tab:4}
\end{table*}
To validate the effectiveness of each component within the Skip-Vision framework, under CoS setting, we conducted ablation and comparative experiments on the skip FFN, summary token, token merge, last visual token, and skip KV-cache. The detailed experimental results are presented in the table \ref{tab:3}.

\textbf{The last summary token or final visual token must pass through the FFN.} As discussed in Section \ref{skip-ffn}, the final visual token plays a crucial role in predicting the subsequent text token, thereby requiring access to the textual knowledge encoded within the FFN layers. This integration of information is essential. Compare (2, 3, 4, 5) in Table \ref{tab:3}, employing a summary token as the final visual token has demonstrated enhanced effectiveness compared to a standard visual token. Furthermore, our experimental findings reveal that optimal performance is achieved when both the summary token and the last local visual token are processed through the FFN.

\textbf{The former summary token enhances the model's comprehension of overall visual information.} Compare (4,6), (7,8), (9,10) in Table \ref{tab:3}, the former summary token enhances the emphasis on crucial information by adaptively merging large-scale visual features. This approach addresses the challenge posed by overly lengthy sequences of visual tokens bypassing the FFN, which may result in the omission of critical large-scale visual context.

\textbf{The local token merge strategy seamlessly aligns with the skip-vision framework.}  Compare (0,1), (4,7), (6,8) in Table \ref{tab:3}, when loacl token merge is directly applied to the $CoS_{1296}$ baseline, the model's overall performance declines, reflecting its dependency on redundant visual data when all tokens pass through the FFN. In contrast, within the skip-vision framework, merging local tokens results in improved performance, indicating that our architecture efficiently leverages visual information without requiring excessive redundancy.

\textbf{The skip KV-cache mechanism enables adaptive selection of visual tokens to skip based on task-specific information requirements.} As demonstrated in Table \ref{tab:4}, for tasks such as MMB, MMU, and MMStar that do not necessitate fine-grained information, both middle and small window-size visual tokens can be skipped, with only the initial and final summary tokens retained. For detail-sensitive tasks like TextVQA and OCRBench, we skip only the small window-size tokens that bypass the FFN, thereby preserving critical fine-grained details. Applying skip KV-cache with small window-size tokens during inference improves performance, particularly in tasks requiring extended responses, such as MMVet.

\subsection{More visualizations}

\begin{figure}[htbp]
  \centering
  \setlength{\abovecaptionskip}{0.cm}
   \includegraphics[width=\linewidth]{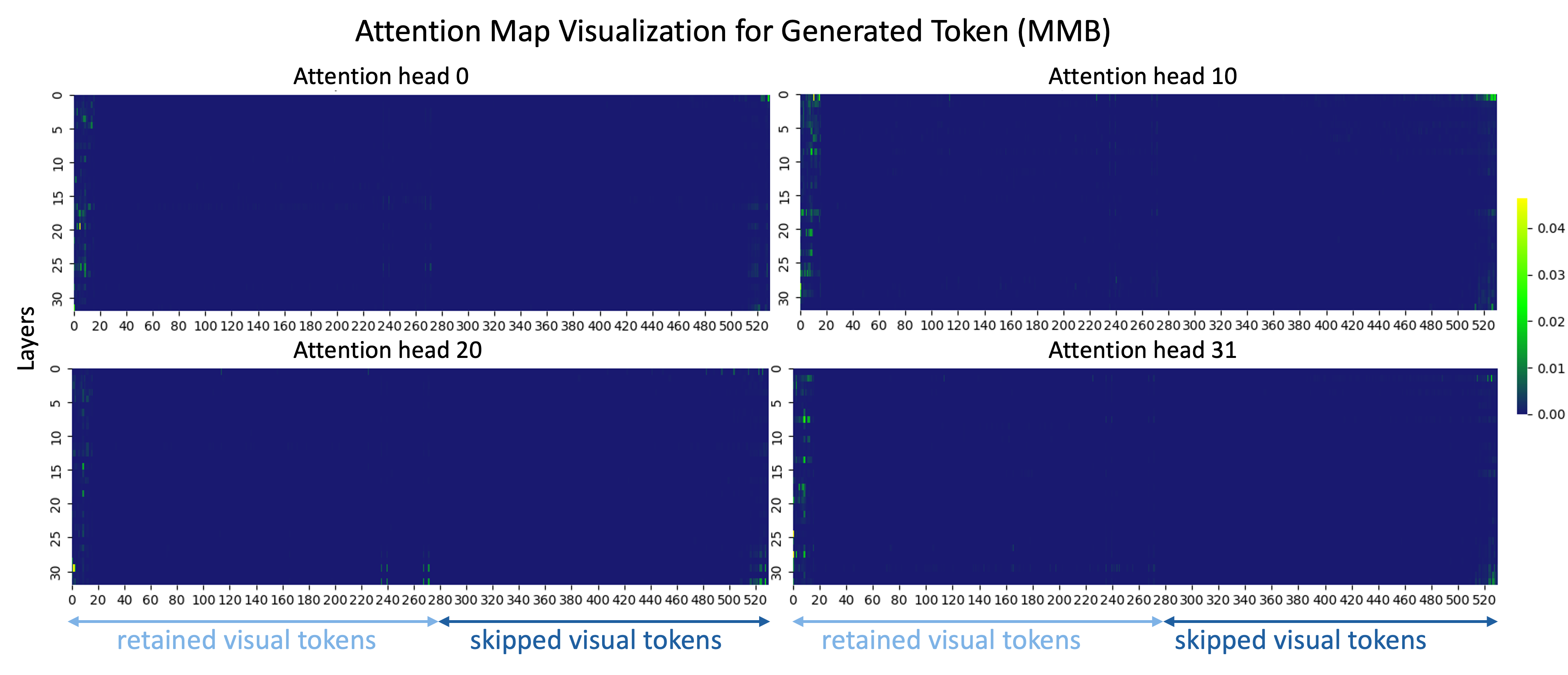}
   \caption{\textbf{Visualization of attention map in MMBench.}}
   \label{fig:mmb}
   \vspace{-5mm}
\end{figure}

\begin{figure}[htbp]
  \centering
  \setlength{\abovecaptionskip}{0.cm}
   \includegraphics[width=\linewidth]{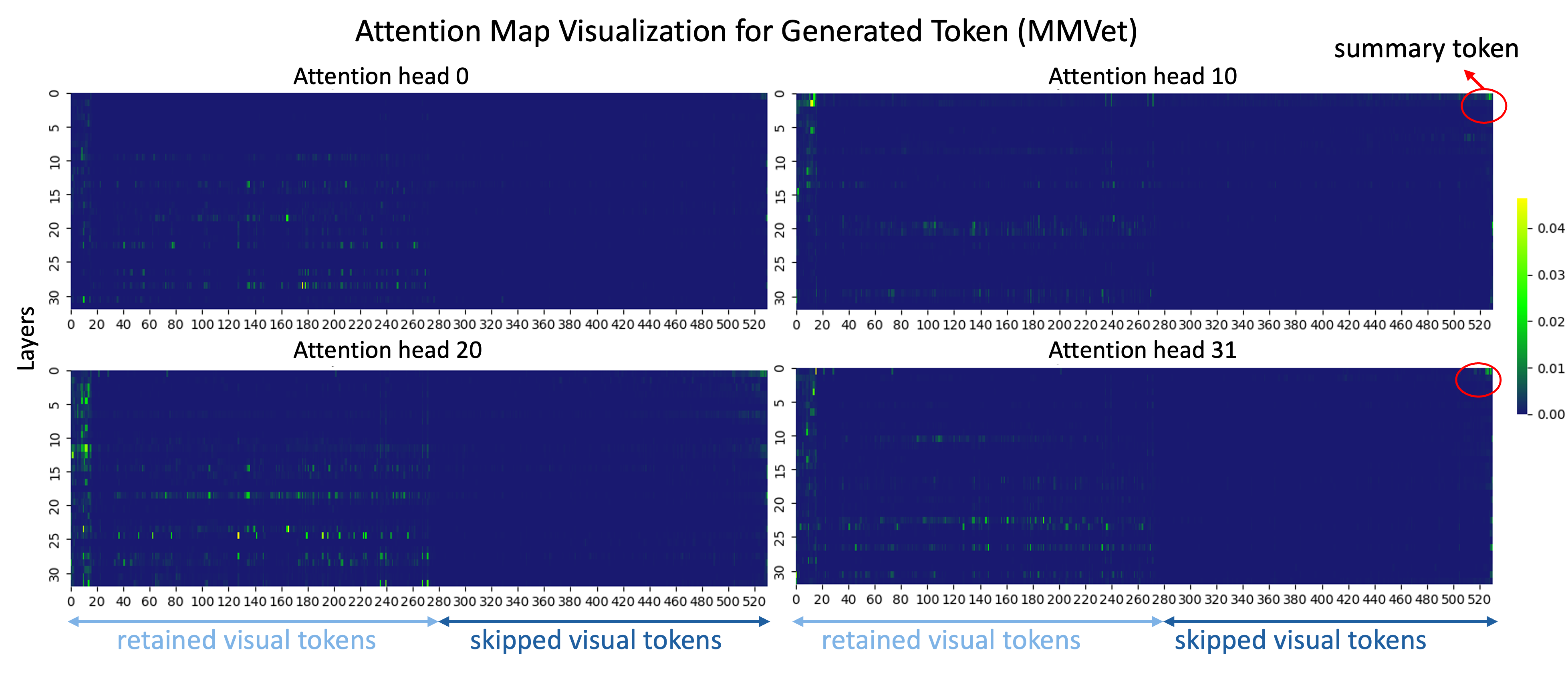}
   \caption{\textbf{Visualization of attention map in MMVet.}}
   \label{fig:mmvet}
   \vspace{-5mm}
\end{figure}

\begin{figure}[htbp]
  \centering
  \setlength{\abovecaptionskip}{0.cm}
   \includegraphics[width=\linewidth]{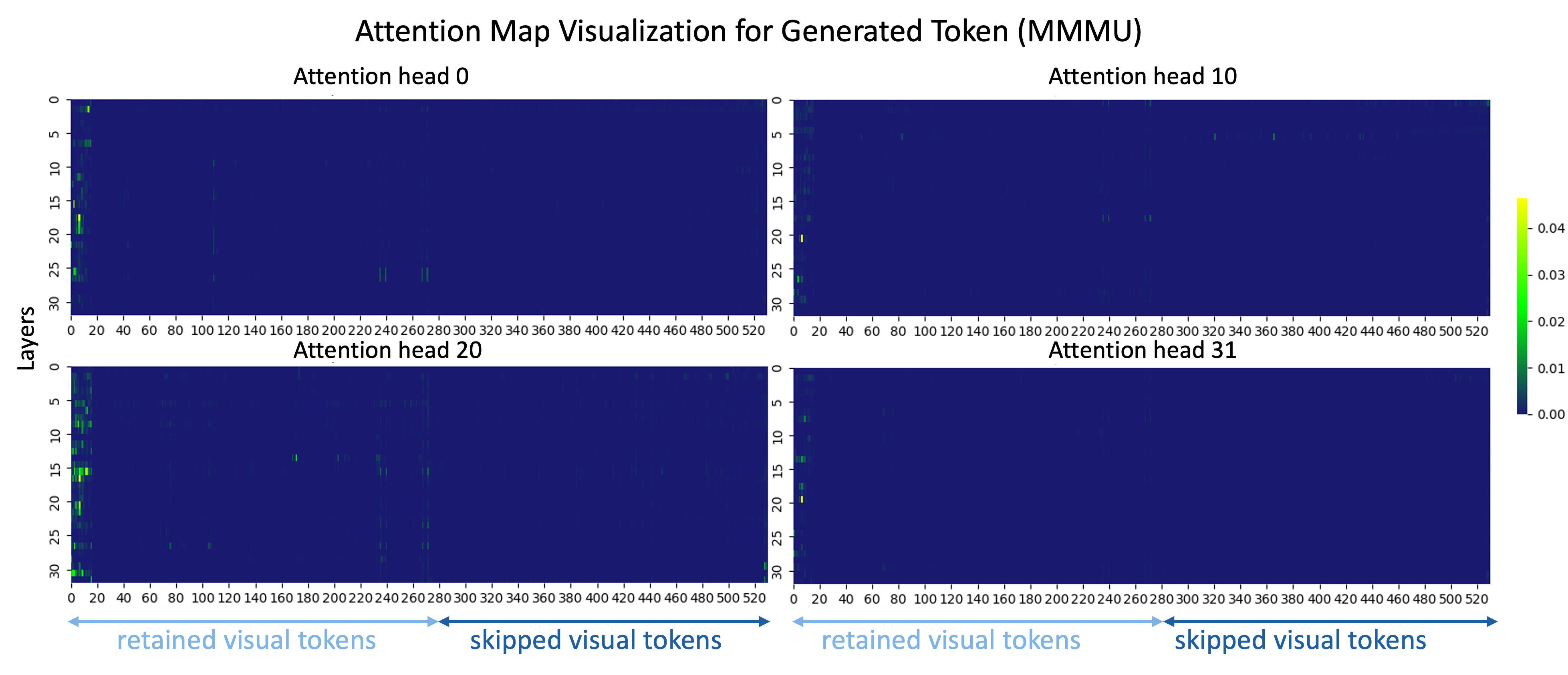}
   \caption{\textbf{Visualization of attention map in MMMU.}}
   \label{fig:mmmu}
   \vspace{-5mm}
\end{figure}

\begin{figure}[htbp]
  \centering
  \setlength{\abovecaptionskip}{0.cm}
   \includegraphics[width=\linewidth]{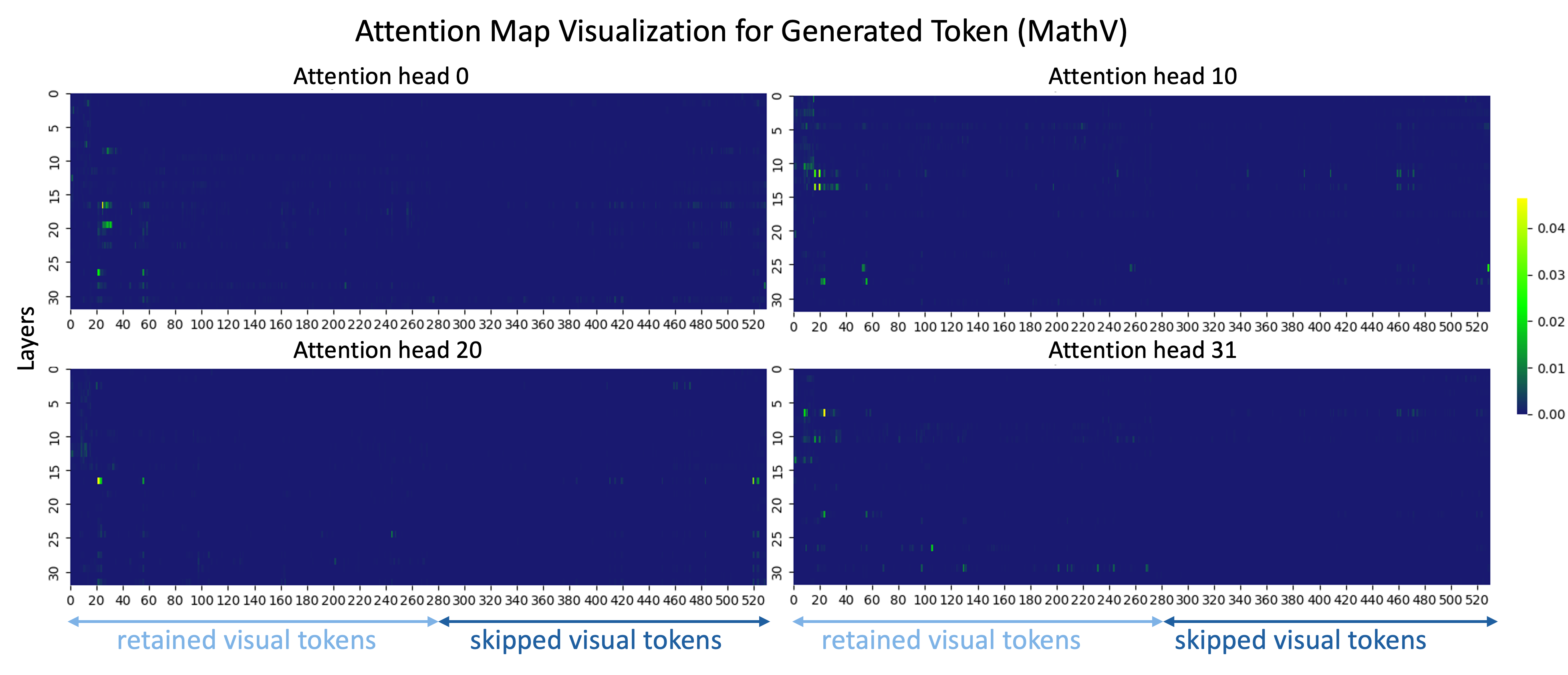}
   \caption{\textbf{Visualization of attention map in MathV.}}
   \label{fig:math}
   \vspace{-5mm}
\end{figure}

\begin{figure}[htbp]
  \centering
  \setlength{\abovecaptionskip}{0.cm}
   \includegraphics[width=\linewidth]{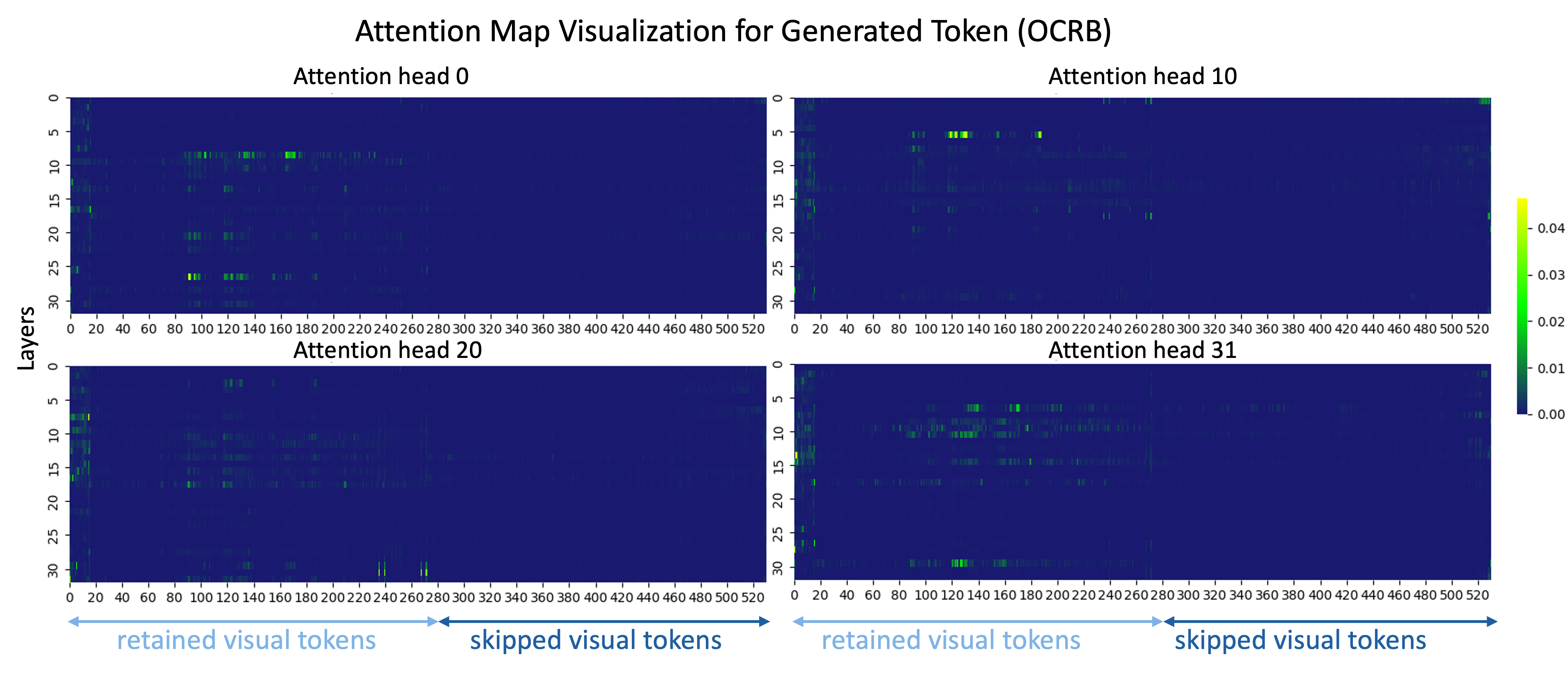}
   \caption{\textbf{Visualization of attention map in OCRBench.}}
   \label{fig:ocr}
   \vspace{-5mm}
\end{figure}

\begin{figure}[htbp]
  \centering
  \setlength{\abovecaptionskip}{0.cm}
   \includegraphics[width=\linewidth]{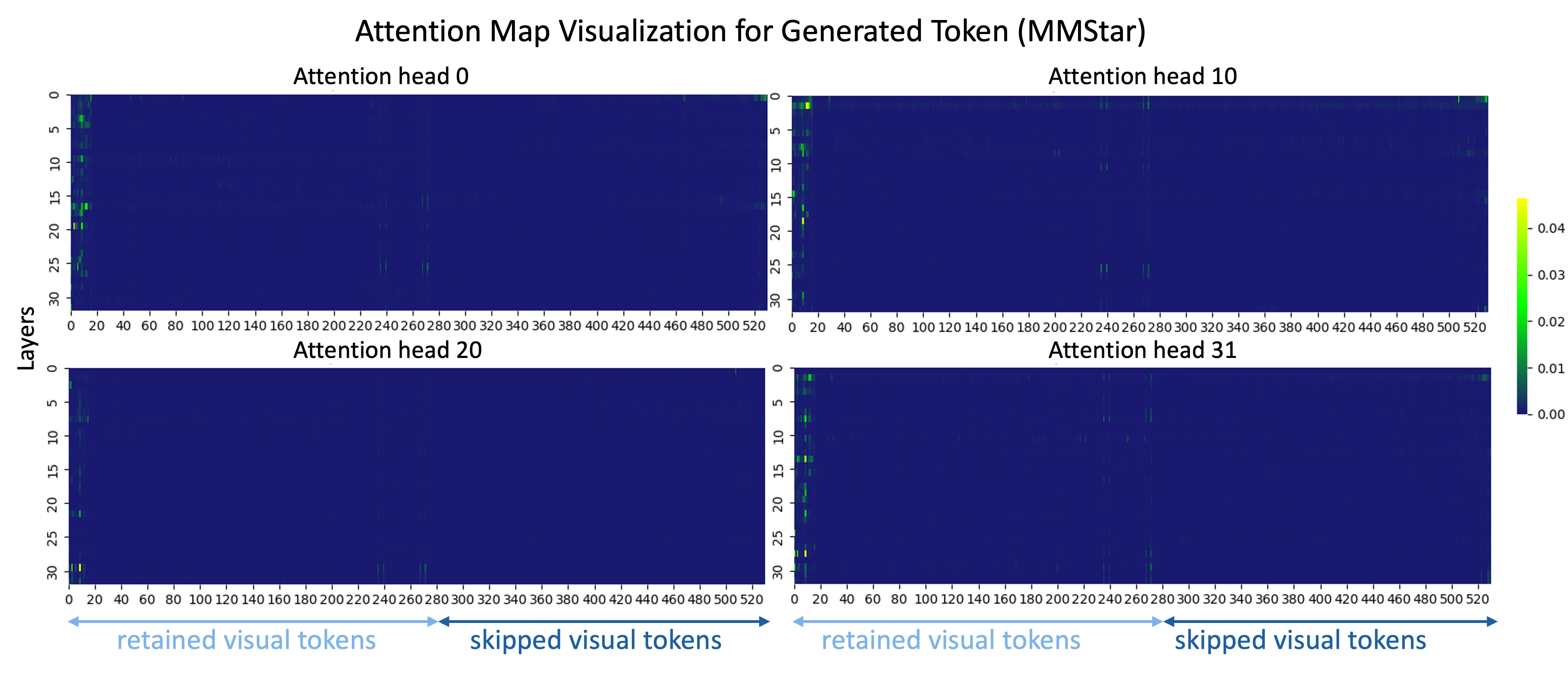}
   \caption{\textbf{Visualization of attention map in MMStar.}}
   \label{fig:mmstar}
   \vspace{-3mm}
\end{figure}

\subsection{Dataset}
\label{sec:dataset}
In Table \ref{tab:SK-1M} and Table \ref{tab:SK-9M}, we introduce the two scaling datasets used by Skip-Vision during the SFT stage: SK-1M and SK-9M.

\clearpage
\begin{table*}[htp]
\centering
\begin{tabular}{p{0.3\linewidth} p{0.6\linewidth}}
\toprule
Task & Dataset \\
\midrule
Visual Instruction Tuning & LLaVA-665k~\cite{llava}, SVIT~\cite{svit}\\
\midrule
VQA & CREPE~\cite{crepe},Imagenet multi task~\cite{imagenet}, VQA-rad~\cite{vqa-rad} \\
\midrule
Visual Reasoning & Wikitable~\cite{wikitable}, Super-CLEVR~\cite{superclevr}, VSR~\cite{vsr} \\
\midrule
Knowledge & ViQuAE~\cite{viquae},Kvqa~\cite{kvqa}, Websrc~\cite{websrc} \\
\midrule
Chart / Diagram / graph & ChartQA~\cite{chartqa},Iconqa~\cite{iconqa}, Infographicvqa~\cite{infographicvqa}, \\
\midrule
Document & DeepForm~\cite{deepform}, TAT-QA~\cite{tat}, Visualmrc~\cite{visualmrc}, Docvqa~\cite{docvqa}, Sujet-Finance-QA-Vision-100k~\cite{Sujet-Finance-QA-Vision-100k} \\
\midrule
Math & Mathverse~\cite{mathverse} \\
\midrule
OCR / Screen / Scene text & TQA~\cite{tqa}, HW-SQuAD~\cite{squad}, TextVQA~\cite{textvqa}, ST-VQA~\cite{stvqa},TextOCR-GPT4V~\cite{textocr-gpt4v}, OCRbench-kv~\cite{ocrbench-kv}, Uber-text~\cite{Uber-text} \\
\midrule
Science & AI2D~\cite{ai2d} \\
\bottomrule
\end{tabular}
\caption{Datasets used by SK-1M at the SFT stage.}
\label{tab:SK-1M}
\end{table*}

\begin{table*}[htp]
\centering
\begin{tabular}{p{0.3\linewidth} p{0.6\linewidth}}
\toprule
Task & Dataset \\
\midrule
Visual Instruction Tuning & SVIT~\cite{svit}, ALLaVA~\cite{allava}, ShareGPT4V~\cite{sharegpt4v}, cog-vlm-sft~\cite{cogvlm}\\
\midrule
Caption & TextCaps~\cite{textcaps}, ShareGPT-4o~\cite{sharegpt4o} \\
\midrule
VQA & CREPE~\cite{crepe},Imagenet multi task~\cite{imagenet}, VQA-rad~\cite{vqa-rad}, VQAv2~\cite{vqav2}, Vizwiz~\cite{vizwiz} \\
\midrule
Visual Reasoning & Wikitable~\cite{wikitable}, Super-CLEVR~\cite{superclevr}, VSR~\cite{vsr}, FigureQA~\cite{figureqa}, TallyQA~\cite{tallyqa}, Visual cot~\cite{visualcot}, CLEVR~\cite{clevr}, Raven~\cite{raven} \\
\midrule
Knowledge & ViQuAE~\cite{viquae},Kvqa~\cite{kvqa}, Websrc~\cite{websrc}, OK-VQA~\cite{okvqa}, Volcano~\cite{volcano}, RLAIF-V~\cite{rlaifv} \\
\midrule
Chart / Diagram / graph & ChartQA~\cite{chartqa},Iconqa~\cite{iconqa}, Infographicvqa~\cite{infographicvqa}, MapQA~\cite{mapqa}, TabFact~\cite{tabfact}, Chart2Text~\cite{charttext}, DVQA~\cite{dvqa}, Chartbench~\cite{chartbench}, MMC~\cite{mmc} \\
\midrule
Document & DeepForm~\cite{deepform}, TAT-QA~\cite{tat}, Visualmrc~\cite{visualmrc}, Docvqa~\cite{docvqa}, Sujet-Finance-QA-Vision-100k~\cite{Sujet-Finance-QA-Vision-100k}, Docmatix~\cite{docmatix}, DocReason25K~\cite{docreason}, DocStruct4M~\cite{DocStruct4M} \\
\midrule
Math & Mathverse~\cite{mathverse}, MathOCR~\cite{mathocr}, MathV360K~\cite{mathv360k}, GeoGPT4V~\cite{geogpt4v}, Geo170K QA~\cite{geo170k} \\
\midrule
OCR / Screen / Scene text & TQA~\cite{tqa}, HW-SQuAD~\cite{squad}, TextVQA~\cite{textvqa}, ST-VQA~\cite{stvqa},TextOCR-GPT4V~\cite{textocr-gpt4v}, OCRbench-kv~\cite{ocrbench-kv}, Uber-text~\cite{Uber-text}, OCR-VQA~\cite{ocrvqa}, ScreenQA~\cite{screenai}, SynthText~\cite{synthetic}, ChromeWriting~\cite{chromewriting}, K12 Printing~\cite{k12}, SQuAD~\cite{squad-v2}, ICDAR19-LSVT~\cite{icdar}, ICPR18-MTWI~\cite{icpr}, ICDAR19-ArT~\cite{icdar-art}, COCO-Text~\cite{coco-text}, Docscan~\cite{Docscan}, HierText~\cite{hiertext}\\
\midrule
Science & AI2D~\cite{ai2d}, Plotqa~\cite{plotqa}, ArXivQA~\cite{arxivvqa}, ScienceQA~\cite{scienceqa} \\
\bottomrule
\end{tabular}
\caption{Datasets used by SK-9M at the SFT stage.}
\label{tab:SK-9M}
\end{table*}

\end{document}